\documentclass[11pt]{article}

\usepackage{parskip}
\usepackage{paralist}
\usepackage{enumerate}
\usepackage{booktabs} 

\usepackage[utf8]{inputenc}
\usepackage{geometry}
\newcommand{\G}{\mathcal{G}}
\newcommand{\Z}{\mathcal{Z}}

\usepackage{graphpap,amscd,mathrsfs,graphicx,lscape,enumitem,dsfont,bm,url,subfigure}
\usepackage{epsfig,amstext,xspace}
\usepackage{algorithm,algorithmic,comment}
\usepackage{enumerate}
\usepackage{thmtools,thm-restate}

\usepackage{enumerate}
\usepackage{algorithm,algorithmic}
\usepackage{auxiliary}
\usepackage[utf8]{inputenc} 
\usepackage[T1]{fontenc}    
\usepackage{hyperref}       
\usepackage{url}            
\usepackage{booktabs}       
\usepackage{amsfonts}       
\usepackage{nicefrac}       
\usepackage{microtype}      
\usepackage{setspace}

\usepackage{color}              
\usepackage[suppress]{color-edits}
\addauthor{tl}{cyan}
\addauthor{et}{blue}
\addauthor{ks}{brown}

\newcommand{\shortversion}[1]{} 

\newcommand{\F}{\mathcal{F}}

\begin{document}

\title{
Small-loss bounds for online learning with partial information
}

 \author{
 Thodoris Lykouris \thanks{Massachusetts Institute of Technology, \texttt{lykouris@mit.edu}. Work mostly conducted while the author was a Ph.D. student at Cornell University and was supported under NSF grant CCF-1563714.}
 \and
 Karthik Sridharan
 \thanks{Cornell University, \texttt{sridharan@cs.cornell.edu}. Work supported in part by NSF grant CDS\&E-MSS 1521544.}
 \and
 \'{E}va Tardos\thanks{Cornell University, \texttt{eva@cs.cornell.edu}. Work supported in part by NSF grant CCF-1563714.}
}

\date{First version: November 2017\\Current version: July 2021%
\footnote{The current version represents the content that will appear in Mathematics of Operations Research. An extended abstract of the paper appeared at the 31st Annual Conference on Learning Theory (COLT 2018).}}

\maketitle
\thispagestyle{empty}

\begin{abstract}
We consider the problem of adversarial (non-stochastic) online learning with partial information feedback, where at each round, a decision maker selects an action from a finite set of alternatives. We develop a black-box approach for such problems where the learner observes as feedback only losses of a subset of the actions that includes the selected action. When losses of actions are non-negative, under the graph-based feedback model introduced by Mannor and Shamir, we offer algorithms that attain the so called ``small-loss'' $o(\alpha L^{\star})$ regret bounds with high probability, where $\alpha$ is the independence number of the graph, and $L^{\star}$ is the loss of the best action. Prior to our work, there was no data-dependent guarantee for general feedback graphs even for pseudo-regret (without dependence on the number of actions, i.e. utilizing the increased information feedback). Taking advantage of the black-box nature of our technique, we extend our results to many other applications such as semi-bandits (including routing in networks), contextual bandits (even with an infinite comparator class), as well as learning with slowly changing (shifting) comparators.

In the special case of classical bandit and semi-bandit problems, we provide optimal small-loss,  high-probability guarantees of $\widetilde{\bigO}(\sqrt{dL^{\star}})$ for actual regret, where $d$ is the number of actions, answering open questions of Neu.  Previous bounds for bandits and semi-bandits were known only for pseudo-regret and only in expectation. We also offer an optimal $\widetilde{\bigO}(\sqrt{\kappa L^{\star}})$ regret guarantee for fixed feedback graphs with clique-partition number at most $\kappa$.
\end{abstract}
\newpage
\setcounter{page}{1}

\section{Introduction}
\label{sec:introduction}
The online learning paradigm \cite{Littlestone1994,prediction_book} has become a key tool for solving a wide spectrum of problems such as  developing strategies for players in large multiplayer games \cite{Blum2006,Blum2008,Roughgarden15,LykourisST16,FosterLLST16},  designing online marketplaces and auctions \cite{BlumHartline,Cesa-Bianchi2013,Roughgarden:2016MRM}, portfolio investment \cite{Cover91,Freund97,HazanKaleAurora07}, online routing \cite{AwerbuchK04,KalaiVempala}. In each of these applications, the learner has to repeatedly select an action on every round. Different actions have different costs or losses associated with them on every round. The goal of the learner is to minimize her cumulative loss and the performance of the learner is evaluated by the notion of ``\emph{regret}'', defined as the difference between the cumulative loss of the learner, and the cumulative loss $L^\star$ of the benchmark. 

The term ``\emph{small-loss} regret bound'' is often used to refer to bounds on regret that depend (or mostly depend) on $L^\star$, rather than the total number of rounds played $T$ often referred to as the time horizon. For instance, for many classical online learning problems, one can in fact show that regret can be bounded by $\widetilde O(\sqrt{L^\star})$ rather than $\widetilde O(\sqrt{T})$. However, these algorithms use the \emph{full information} model:  assume that on every round, the learner receives as feedback the losses of all possible actions (not only the selected actions). In such full information settings, it is well understood when small-loss bounds are achievable and how to design learning algorithms that attain them. However, in most applications, full information about losses of all actions is not available. Unlike the full information case, the problem of obtaining small-loss regret bounds for partial information settings is poorly understood. Even in the classical multi-armed bandit problem, small-loss bounds are only known in expectation against the so called oblivious adversaries or comparing against the lowest expected cost of an arm (and not the actual lowest cost), referred to as pseudo-regret.

The goal of this paper is to develop robust techniques for extending the small-loss guarantees to a broad range of partial-feedback settings where learner only observes losses of selected actions and some neighboring actions. In the basic online learning model, at each round $t$, the decision maker or \emph{learner} chooses one action from a set of $d$ actions, typically referred to as \emph{arms}. Simultaneously an adversary picks a loss vector $\ell^t\in[0,1]^d$ indicating the losses for the $d$ arms. The learner suffers the loss of her chosen arm and observes some feedback. The variants of online learning differ by the nature of feedback received. The two most prominent such variants are the \emph{full information setting}, where the feedback is the whole loss vector, and the \emph{bandit setting} where only the loss of the selected arm is observed. Bandits and full information represent two extremes. In most realistic applications, a learner  choosing an action $i$, learns not only the loss $\ell_i^t$ associated with her chosen action $i$, but also some partial information about losses of some other actions. A simple and elegant model of this partial information is the \emph{graph-based} feedback model of \cite{MannorS11,AlonCGMMS}, where at every round, there is a (possibly time-varying) undirected graph $G^t$ representing the information structure, where the possible actions are the nodes. If the learner selects an action $i$ and incurs the loss $\ell_i^t$, she observes the losses of all the nodes connected to node $i$ by an edge in $G^t$. Our main result in Section \ref{sec:black_box} is a general technique that allows us to use any full information learning algorithm as a black-box, and design a learning algorithm whose regret can be bounded with high probability as $o(\alpha L^{\star})$, where $\alpha$ is the maximum independence number of the feedback graphs. This graph-based information feedback model is a very general setting that can encode all of full information, bandit, as well as a number of other applications. 

\subsection{Our contribution}

\paragraph{Our results.} We develop a unified, black-box technique to achieve small-loss regret guarantees with high probability in various partial-feedback models. We obtain the following results.
\begin{compactitem}
\item  In Section \ref{sec:black_box}, we provide a generic black-box reduction from any small-loss full information algorithm. When used with known algorithms it achieves regret guarantees of 
$\widetilde{\bigO}\prn*{(L^{\star})^{2/3}}$ 
that hold with high probability for any of multi-armed bandits, combinatorial semi-bandits, contextual bandits, or feedback graphs (with dependence on the information structure in the $\widetilde{\bigO}$ as $d^{1/3}$ for the first three, and $\alpha^{1/3}$ for feedback graphs). There are three novel features of this result. First, unlike most previous work in partial information that is heavily algorithm-specific, our technique is black-box in the sense that it takes as input a small-loss full information algorithm and, via a small modification, makes it work under partial information. Second, prior to our work, there was no data-dependent guarantee for general feedback graphs even for pseudo-regret (without dependence on the number of actions, i.e., taking advantage of the increased information feedback),  while we provide a high probability small-loss guarantee. Last, our guarantees are not for pseudo-regret but are rather regret guarantees that hold with high probability. 
\item In Section \ref{sec:optimal_lst_highprob}, we combine our black-box framework with specific algorithms to achieve an optimal dependence on $L^\star$ for important feedback structures. We first show an algorithm with an optimal high-probability guarantee for the multi-armed bandit setting (Section~\ref{ssec:optimal_bandit}). This improves on previous work that only applies to the weaker notion of pseudo-regret \cite{Allenberg2006,FosterLLST16}. Subsequently, we extend this guarantee to attain an optimal dependence with respect to $L^{\star}$ for general graph-based feedback (Section~\ref{ssec:clique_partition}). Two caveats of the last result is that it scales with the clique-partition number (number of parts needed to partition the feedback graph into complete graphs) instead of the independence number (maximum size of an independent set) and that it requires the graph to be fixed across time. Improving on either of these fronts is an interesting open direction.
\item In Section~\ref{sec:applications}, we show the wide applicability of our framework by deriving results for combinatorial semi-bandits (Section \ref{ssec:semi_bandits_black_box}), computationally efficient contextual bandits (Section~\ref{ssec:contextual_bandits_black_box}), as well as learning with slowly changing (shifting) comparators (Section~\ref{ssec:shifting_comparators}). For combinatorial semi-bandits \cite{AudibertBubLug2014} which captures settings such as online routing, we provide in particular the first optimal high-probability small-loss guarantee answering an open question of \cite{Neu15_semibandits,Neu2015_implicit}; this improves on the pseudo-regret guarantees of prior work. For computationally efficient contextual bandits \cite{AuerCeFrSc03,Langford2007}, we do not provide a purely small-loss guarantee since such a guarantee is not known even for full feedback but we replace some of the dependence on the time horizon $T$ by $L^{\star}$. Providing a purely small-loss guarantee still remains an open problem \cite{pmlr-v65-agarwal17a}. At the absence of computational considerations, our work provides a $\widetilde{\bigO}((L^\star)^{2/3})$ small-loss guarantee which has been later improved \cite{AllenZBL2018} to $\widetilde{\bigO}(\sqrt{L^{\star}})$ for the weaker notion of pseudo-regret. Finally, our guarantees for shifting comparators have important game-theoretic implications regarding the efficiency of learning outcomes in games with dynamic population.
\item Our results are derived by first obtaining a bound on a multiplicative approximation of regret; the eventual guarantees come via tuning appropriately the approximation parameter. Beyond serving as building blocks, these approximate regret guarantees have merit on their own right as they allow convergence to efficient outcomes in dynamic population games without requiring knowledge of the number of changes in the comparator sequence (see discussion in Section~\ref{ssec:shifting_comparators}).
\end{compactitem}

\paragraph{Our techniques.}
Classical partial-information guarantees are based on creating an unbiased estimator for the loss of each arm and then running a full information algorithm on the estimated losses. The most prominent such unbiased estimator, called \emph{importance sampling}, is equal to the actual loss divided by the probability with which the action is played. This division can make the estimated losses unbounded in the absence of a lower bound on the probability of being played. Algorithms like EXP3 \cite{AuerCeFrSc03} for the bandit setting or Exp3-DOM \cite{AlonCGMMS} for the graph-based feedback setting mix in a $1/\sqrt{T}$ amount of uniform action distribution
which ensures that the range of losses is bounded. This step necessarily leads to a dependence on the performance of the worst action, as regardless how poor it may be, we still select it with enough probability at every round. As a result, even when the best arm has small loss, the performance of the algorithm scales with the loss of the worst arm; this approach therefore leads to a
$\bigO{(\sqrt{T})}$ regret. 
Some specialized algorithms do achieve small-loss bounds for bandits, but these techniques extend neither to graph feedback nor to high probability guarantees (see also the discussion below about related work).

Instead of mixing with uniform action distribution, we temporarily avoid playing suboptimal actions (those with probability less than some threshold $\gamma$).
This \emph{freezing} technique was originally introduced by Allenberg et al. \cite{Allenberg2006} with a single threshold $\gamma$ offering a new way to adapt the exponential weights algorithm to the bandit setting. The resulting estimator is negatively biased for the arms that are frozen but is always unbiased for the selected arm. Using these expectations, the regret bound of the full information algorithm can be used to bound the expected regret compared to the expected loss of any fixed arm, achieving low pseudo-regret in expectation. To achieve good bounds, we need to guarantee that the total probability frozen is limited. By freezing arms with probability less than $\gamma$, the total probability that is frozen at each round is at most $d\gamma$ and therefore contributes to a regret term of $d\gamma$ times the loss of the algorithm which gives a dependence on $d$ on the regret bound. This was analyzed in the context of exponential weights in \cite{Allenberg2006}.

Our main technical contribution is to greatly expand the power of this freezing technique. We show how to apply it in a black-box manner with any full information learning algorithm and extend it to graph-based feedback. To deal with the graph-based feedback setting, we suggest a novel and technically more challenging double-threshold freezing scheme (Section~\ref{sec:black_box}). The natural way to apply importance sampling in the graph-based feedback is by dividing the actual loss with the probability of being observed, i.e. the sum of the probabilities that the action and its neighbors are played. An initial approach is to freeze an action if its probability of being observed is below some threshold $\gamma$. We show that the total probability frozen by this step is bounded by $\alpha\gamma$, where $\alpha$ is the size of the maximum number of nodes in an independent set, a subset of nodes of the feedback graph with no edges. To see why, consider a maximal independent set $S$ of the frozen actions and note that all frozen actions are observed by some node in $S$. This observation seems to imply that we can replace the dependence on $d$ by a dependence on $\alpha$. However there are externalities among actions as freezing one action may affect the probability of another being observed. As a result, the latter may need to be frozen as well to ensure that all active arms are observed with probability at least $\gamma$ (and therefore obtain our desired upper bound on the range of the estimated losses). This causes a cascade of freezing, with possibly freezing a large amount of additional probability. To limit this cascade effect, we develop a double-threshold freezing technique: we initially freeze arms that are observed with probability less than $\gamma$, and subsequently use a lower threshold $\gamma'=\gamma/3$ and only freeze arms that are observed with probability less than $\gamma'$. This technique allows us to bound the total probability of arms that are frozen subsequently by the total probability of arms that are frozen initially. We prove this via an elegant combinatorial charging argument of Claim~\ref{Claim_2}. 

To go beyond pseudo-regret and guarantee regret bounds with high probability, it does not suffice to have the estimator be negatively biased but we need to also obtain a handle on the variance. We prove that freezing also provides such a lever leading to a high-probability $\widetilde{\bigO}{(\alpha^{1/3}\prn*{L^{\star}}^{2/3})}$ regret guarantee that holds in a black-box manner. Interestingly, this freezing technique via a small modification enables the same guarantee for semi-bandits where the independent set is replaced by the number of elements (edges). Finally, in order to obtain the optimal high-probability guarantee for bandits and semi-bandits (Sections~\ref{ssec:optimal_bandit} and \ref{ssec:semi_bandits_black_box}), we need to combine our black box analysis with taking advantage of features of concrete full information learning algorithms. The black-box nature of the previous analysis is extremely useful in demonstrating where additional features are needed. 

\subsection{Related work}
\label{ssec:related_work}
Online learning with partial information dates back to the seminal work of Lai and Robbins \cite{Lai1985}. They consider a stochastic version, where losses come from fixed distributions. The case where the losses are selected adversarially, i.e. they do not come from a distribution and may be adaptive to the algorithm's choices, which we examine in this paper, was first studied by Auer et al. \cite{AuerCeFrSc03} who provided the EXP3 algorithm for multi-armed bandits and the EXP4 algorithm for learning with expert advice (a more general model than contextual bandits considered in \cite{Langford2007}). They focus on uniform regret bounds, i.e. that grow as a function of time $o(T)$, and bound mostly the expected performance, but such guarantees can also be derived with high probability \cite{AuerCeFrSc03,AudibertB10,pmlr-v15-beygelzimer11a}. When the performance of arms is evaluated based on non-negative rewards rather than non-negative losses, regret guarantees that depend on the total reward $R^{\star}$ of the best arm are easily derived from the above algorithms as even getting reward $0$ with probability of $\epsilon$ only contribute $\epsilon R^{\star}$ in the regret. In contrast, incurring a loss of $1$ with probability $\epsilon$ can be arbitrarily larger than $\epsilon L^{\star}$ when the best arm has very small loss. In this paper we develop such guarantees that depend on $L^{\star}$ for partial information algorithm for the cases of losses. There are a few specialized algorithms that achieve such small-loss guarantees for the case of bandits for pseudo-regret, e.g. by ensuring that the estimated losses of all arms remain close \cite{Allenberg2006,Neu15_semibandits} or using a stronger regularizer \cite{RakhlinS13predictablesequences,FosterLLST16}, but all of these methods neither offer high probability small-loss guarantees even for the bandit setting, nor extend to graph-based feedback. Beyond bandits, small-loss guarantees are also achieved again for pseudo-regret for label-efficient prediction \cite{CesaLuSt05}, combinatorial semi-bandits \cite{Neu15_semibandits}, and subsequently to our work computationally inefficient contextual bandits via a a more sophisticated use of freezing suboptimal arms \cite{AllenZBL2018}. For bandits and semi-bandits, our technique allows us to derive the first optimal high-probability small-loss bound (extending beyond pseudo-regret).

The graph-based partial information that we examine in this paper was introduced by Mannor and Shamir \cite{MannorS11} who provided ELP, a linear programming based algorithm achieving $\widetilde{\bigO}{(\sqrt{\alpha T})}$ regret for undirected graphs. Alon et al. \cite{AlonCGM13,AlonCGMMS} provided variants of Exp3  (Exp3-SET) that recovered the previous bound via what they call \emph{explicit exploration}. Following this work, there have been multiple results on this setting, e.g.\cite{AlonCBDK15,Cohen2016, KocakNV16, tossou:aaai2017b}, but prior to our work, there was no small-loss guarantee for the feedback graph setting that could exploit the graph structure. To obtain a regret bound depending on the graph structure, the above techniques upper bound the losses of the arms by the maximum loss which results in a dependence on the time horizon $T$ instead of $L^{\star}$. Addressing this, we achieve regret that scales with an appropriate problem dimension, the size of the maximum independent set $\alpha$, instead of ignoring the extra information and only depending on the number of arms as all small-loss results of prior work.

Biased estimators have been used prior to our work for achieving better regret guarantees and have been used, for example, for high-probability guarantees in EXP3.P \cite{AuerCeFrSc03}. The freezing technique of \cite{Allenberg2006} can be thought of as the first use of biased estimators towards small-loss guarantees.  Their \emph{GREEN} algorithm uses freezing in the context of the exponential weights algorithm for the case of multi-armed bandits. Freezing keeps the range of estimated losses bounded and when used with the exponential weights algorithm, also keeps the cumulative estimated losses very close, which ensures that one does not lose much in the application of the full information algorithm. Using these facts Allenberg et al. \cite{Allenberg2006} achieved small-loss guarantees for  pseudo-regret in the classical multi-armed bandit setting. An approach very close to freezing is the \emph{implicit exploration} of Koc\'ak et al. \cite{Kock2014EfficientLB} that adds a term in the denominator of the estimator making the estimator biased, even for the selected arms. The\emph{FPL-TrIX} algorithm of Neu \cite{Neu15_semibandits} is based on the Follow the Perturbed Leader algorithm using implicit exploration together with truncating the perturbations to  guarantee that the estimated losses of all actions are close to each other and the \emph{geometric resampling} technique of Neu and Bart\'ok \cite{NeuB13} to obtain these estimated losses. His analysis provides small-loss regret bounds for pseudo-regret, but does not extend to high-probability guarantees. The \emph{EXP3-IX} algorithm of Koc\'ak et al. \cite{Kock2014EfficientLB} combines implicit exploration with exponential weights to obtain, via the analysis of Neu \cite{Neu2015_implicit}, high-probability uniform bounds. Focusing on uniform regret bounds, implicit exploration and truncation were presented as strictly superior to freezing. In this paper, we show an important benefit of the freezing technique: it can be extended to handle feedback graphs (via our dual-thresholding). We also combine freezing with exponential weights to develop an algorithm we term \emph{GREEN-IX} which achieves optimal high-probability small-loss $\widetilde{\bigO}(\sqrt{d L^{\star}})$ for the multi-armed bandit setting. Finally, combining freezing with the truncation idea, we obtain the corresponding result for semi-bandits; in contrast, it is unclear whether the geometric resampling analysis suffices to provide such a handle on the variance of the estimated loss. 

\section{Model}
\label{sec:model}
In this section we describe the basic online learning protocol and the partial information feedback model we consider in this paper. In the online learning protocol, in each round $t$, the learner selects a distribution $w^t$ over $d$ possible actions, i.e. $w_i^t$ denotes the probability with which action $i$ is selected on round $t$. The adversary then picks losses $\ell^t=(\ell_1^t,\dots, \ell_d^t)$ where $\ell_i^t \in [0,1]$ denotes the loss of action $i$ in round $t$. The learner then draws action $I(t)$ from the distribution $w^t$ and suffers the corresponding loss $\ell^t_{I(t)}$ for that round. In the end of round $t$, the learner receives feedback about the losses of the selected action and some neighboring actions. The feedback received by the learner in each round is based on a feedback graph model described below. 

\subsection{Feedback graph model}
We assume that the learner receives partial information based on an undirected feedback graph $G^t$ that is allowed to vary across rounds. The learner observes the loss $\ell_{I(t)}^t$ of the selected arm $I(t)$ and, in addition, she also observes the losses of all arms connected to the selected arm $I(t)$ in the feedback graph. More formally, she observes the loss $\ell_j^t$ for all the arms $j\in N_{I(t)}^t$ where $N_i^t$ denotes the set containing arm $i$ and all neighbors of $i$ in $G^t$ at round $t$. The \emph{full information feedback setting} and the \emph{bandit feedback setting} are special cases of this model where the graphs $G^t$ are the complete and the empty graphs respectively for all rounds $t$. The graph parameters that are useful in our work are the independence number of the graph $\alpha(G)$, which is the size of the maximum subset of nodes with no edges between them, and the clique-cover number $\kappa(G)$, which is the number of parts needed to partition the feedback graph into complete graphs.

We allow the feedback graph $G^t$ to change each round $t$, but assume that the graph $G^t$ is known to the player before selecting her distribution $w^t$. This model also includes the \emph{contextual bandits} problem of \cite{AuerCeFrSc03,Langford2007} as a special case, where each round the learner is also presented with an additional input $x^t$, the context. In this contextual setting, the learner is offered $d$ policies, each suggesting an action depending on the context, and each round the learner can decide which policy's recommendation to follow. To model this with our evolving feedback graph, we use the policies as nodes, and connect two policies with an edge in $G^t$ if they recommend the same action in the context $x^t$ of round $t$.

\subsection{Regret}
In the adversarial online learning framework, we assume only that losses $\ell_i^t$ are in the range $[0,1]$. The goal of the learner is to minimize the so called regret against an appropriate benchmark.  The traditional notion of regret  compares the performance of the algorithm to the best fixed action $f$ in hindsight. For an arm $f$ we define regret as:
$$
\regret(f)=\sum_{t=1}^T \brk*{\ell_{I(t)}^t-\ell_f^t}
$$
where $T$ is the time horizon. To evaluate performance, we consider regret against the best arm:
$$
\regret=
\max_f \regret(f)
$$
Note that the regrets
$\regret(f)$ and $\regret$ 
are random variables. 

A slightly weaker notion of regret is
the notion of pseudoregret (c.f. \cite{BubeckC12}), that compares the expected performance of the algorithm to the expected loss of any fixed arm $f$, fixed in advance and not in hindsight. More formally, this notion of expected regret is:
$$\pseudoregret=\max_f \En_{I(1)\dots I(t)}\brk*{ \regret(f)}
$$
This is weaker than the expected regret $\En_{I(1)\dots I(t)} \brk*{\regret}=\En_{I(1)\dots I(t)}\brk*{\max_f \regret(f)}$.\footnote{To see the difference, consider $n$ arms that are similar but have high variance. Pseudoregret compares the algorithm's performance against the expected performance of arms, while regret compares against the ``best'' arm depending on the outcomes of the randomness. This difference can be quite substantial, like when throwing $n$ balls into $n$ bins the expected load of any bin is $1$, while the expected maximum load is $\Theta(\log n/\log \log n)$.}.

We aim for an even stronger notion of regret, guaranteeing low regret with high probability, i.e. for all $\delta>0$ with probability $1-\delta$, instead of only in expectation, at the expense of a logarithmic dependence on $1/\delta$ in the regret bound for any fixed $\delta$. Note that any high-probability guarantee concerning $\regret(f)$ for any fixed arm $f$ with failure probability $\delta'$ can automatically provide an overall regret guarantee with failure probability $\delta=d\delta'$.  A high-probability guarantee on low $\regret$ also implies low regret in expectation.\footnote{If the algorithm guarantees regret at most $B\log(1/\delta)$ with probability at least $(1-\delta)$ for any $\delta>0$, then we can obtain the expected regret bound of $O(B)$ by upper bounding it by the integral $\int_0^{\infty}x \cdot\mathbb{P}(\regret>x) dx$.}

\textbf{Small-loss regret bound.}
The goal of this paper is to develop algorithms with \emph{small-loss} regret bounds, where the loss remains small when the best arm has small loss, i.e. when regret depends on the loss of the comparator, and not on the time horizon. To achieve this, we focus on the notion of approximate regret (c.f. \cite{FosterLLST16}), which is a multiplicative relaxation of the regret notion. We define \emph{$\eps$-approximate regret} for a parameter $\eps>0$ as
$$
\apx\prn*{f,\eps} = (1-\eps) \sum_{t=1}^T 
\ell_{I(t)}^t-  \sum_{t=1}^T \ell_f^t
$$
We prove bounds on $\apx\prn*{f,\eps}$ in high probability and in expectation, and then provide small-loss regret bounds by tuning $\eps$ appropriately, an approach that is often used in the literature in achieving classical regret guarantees and is referred to as \emph{doubling trick}. 
Typically, approximate regret bounds depend inversely on the parameter $\eps$. For instance, in classical full information algorithms, the expected approximate regret is bounded by $O\prn{\frac{\log(d)}{\eps}}$ and therefore setting $\eps=\sqrt{\frac{\log(d)}{T}}$, one obtains the classical $O\prn{\sqrt{T\log(d)}}$ uniform bounds. If we knew $L^{\star}$, the loss of the best arm at the end of round $T$, one could set $\eps=\sqrt{\frac{\log(d)}{L^{\star}}}$ and get the desired $O\prn*{\sqrt{L^{\star}\log(d)}}$ guarantee. Of course, $L^{\star}$ is not known in advance, and depending on the model of feedback, may not even be observed either. To overcome these difficulties, we can make the choice of $\eps$ depend on $\widehat{L}$, the loss of the  algorithm instead, and apply doubling trick: start with a large $\eps$, hoping for a small $\widehat{L}$ and halve $\eps$ when we observe higher losses; the latter is known as self-confident approach and was introduced in \cite{Auer2002Adaptive}.

\subsection{Other applications}\label{ssec:appl_models}
\textbf{Combinatorial semi-bandits.}
We also extend our results to a different form of partial information: semi-bandits. In the semi-bandit problem we have a set of \emph{elements} $\prn*{\mathcal{E}}$, such as edges in a network, and the learner needs to select from a set of possible actions $\prn*{\mathcal{F}}$, where each possible action $f \in \mathcal{F}$  corresponds to a subset of the elements $\mathcal{E}$. An example is selecting a path in a graph, where at round $t$, each element $e\in \mathcal{E}$ has a delay $\ell^t_e$, and the learner needs to select a path $P$ (connecting her source to her destination), and suffers the sum of the losses $\sum_{e\in P} \ell_e^t$. We use $\ell_P^t=\sum_{e\in P} \ell_e^t$ as the loss of the strategy $P$ at time $t$. We assume that the learner observes the loss on all edges $e\in P$ in her selected strategy, but does not observe other losses. We measure regret compared to the best single strategy $f\in \mathcal{F}$ with hindsight, so use $\mathcal{F}$ as the set of (possibly exponentially many) \emph{comparators}.

\textbf{Contextual bandits.}
Another important application is the contextual bandit problem, where the learner has a set of $\mathcal{A}$ actions to choose from, but each step $t$ also has a context: At each time step $t$, she is presented with a context $x_t\in\X$, and can base her choice of action on the context. She also has a set $\mathcal{F}$ policies where each $f\in \mathcal{F}$ if a function $f_i(x)\in \mathcal{A}$ from contexts to actions. As an example, actions can be a set of medical treatment options, and contexts are the symptoms of the patient. A possible policy class $\mathcal{F}$ can be finite given explicitly, or large and only implicitly given, or even can be an infinite class of possible policies. 

\textbf{Regret with shifting comparators.} In studying learning in changing environments \cite{Herbster1998}, such as games with dynamic populations \cite{LykourisST16}, it is useful to have regret guarantees against not only a single best arm, but also against a sequence of comparators, as changes in the environment may change the best arm over time. We overload $f$ to denote the vector of the comparators $(f(1),\dots,f(T))$ in such settings. If the comparator changes too often, no learning algorithm can do well against this standard. We consider sequences where $f$ has only a limited number of changes, that is $f(t)=f(t+1)$ for all but $k$ rounds (with $k$ not known to the algorithm).
To compare the performance to a sequence of different comparators, we
extend our regret notions to this case:
$$
\apx\prn*{f,\eps}=\sum_{t=1}^T
    \brk*{(1-\epsilon)\ell_{I(t)}^t-\ell^t_{f(t)}}
$$
where $\epsilon$ corresponds to the multiplicative factor that comes in the regret relaxation. Typically the approximate regret depends linearly on the number of changes in the comparator sequence.

\section{The black-box reduction for graph-based feedback}
\label{sec:black_box}
In this section, we present our black-box framework (outlined in Section~\ref{sec:introduction}) that turns any full-information small-loss learning algorithm into an algorithm with a high-probability small-loss guarantee in the partial-feedback setting. To deal with partial feedback, our algorithm builds on the importance-weighted sampling that divides the loss of an observed arm $\ell_i^t$ by the probability that it is observed $W_i^t$ (step~\ref{step:importance_weight_black_box} in Algorithm~\ref{alg:black_box}). The classical way to get a handle on the variance of this estimator is via mixing with a uniform distribution \cite{AuerCeFrSc03}; however, this incurs dependence on the loss of the \emph{worst} rather than the \emph{best} arm, thereby leading to uniform rather than small-loss bounds. Our approach is based on an improved version of this idea initially proposed by Allenberg et al. \cite{Allenberg2006} in the context of Exponential Weights who derived a pseudo-regret small-loss guarantee for the bandit feedback setting. 

At each round $t$, we select a subset of the arms that are considered suboptimal and neither plays nor updates their loss. We refer to such arms as (temporarily) frozen and note that frozen arms may get unfrozen in later rounds if other arms incur losses. To select this subset, we first freeze arms that are observed with probability less than some threshold $\gamma$ (step~\ref{step:initial_black_box} in Algorithm~\ref{alg:black_box}) and set their estimated loss to $0$ (step~\ref{step:importance_weight_black_box} in Algorithm~\ref{alg:black_box}). The goal behind this step is to guarantee an upper bound of the variance of the estimator in a way that scales with the loss of the algorithm rather than the worst arm. Unfortunately, freezing an arm in turn decreases the probability that the neighbors are observed. This effect can propagate and cause additional arms to be observed with probability less than $\gamma$, violating the upper bound on the estimated loss. To ensure a lower bound on non-frozen arms, we recursively freeze all arms whose observation probability is smaller than a smaller threshold $\gamma'=\gamma/3$ (Step~\ref{step:recursive_black_box} in Algorithm~\ref{alg:black_box}). 
\begin{algorithm}[!h]
\caption{Double-Threshold Freezing Algorithm}
\label{alg:black_box}
\begin{algorithmic}[1]
\REQUIRE Full information algorithm $\mathcal{A}$, an upper bound on the size of maximum independent sets $\alpha$, number of arms $d$, learning parameter $\epsilon'$, freezing thresholds $\gamma:=\frac{\epsilon'}{4\alpha}$ and $\gamma':=\frac{\gamma}{3}$.
\STATE Initialize $p_i^1$ for arm $i$ based on the initialization of $\mathcal{A}$ and set $t=1$ (round 1).
\FOR{$t=1$ \TO $T$}
\STATE Initial step: Freeze all arms whose observation probability is below $\gamma$ to obtain:
$$
F_0^t= \crl*{i: \sum_{j \in N_i^t} p_j^t<\gamma}
$$ \label{step:initial_black_box}
\STATE Propagation steps: Recursively freeze remaining arms if their probability of being observed by non-frozen arms is below $\gamma'$ to obtain 
$F^t=\bigcup_{k\geq 0}F_{k}^t$ where, 
$$
F_k^t= \crl*{i\notin \prn*{\bigcup_{m=0}^{k-1} F_m^t} : \sum_{j\in \prn*{N_i^t\setminus\bigcup_{m=0}^{k-1} F_m^t}} p_j^t < \gamma'
}
$$ \label{step:recursive_black_box}
\STATE Normalize the probabilities of non-frozen arms so that they form a distribution.$$
w_i^t=\begin{cases}
0 & \text{ if } i \in F^t \\
\frac{p_i^t}{1-\sum_{j\in 
F^t} p_j^t} & \text{ else }
\end{cases}
$$  \label{step_normalize_black_box}
\STATE Draw arm $I(t)\sim w^t$ and incur loss $\ell_{I(t)}^t$.
\STATE Compute estimated loss:

$$
\widetilde{\ell}_i^t=\begin{cases}\frac{\ell_i^t}{W_i^t} & \text{ if } i\in N_{I(t)}^t \backslash  F^t\\
0 & \text{ else }
\end{cases}.
$$
where $W^t_i=\sum_{j \in N^t_i} w_j^t$ \label{step:importance_weight_black_box}
\STATE Update $p_i^{t+1}$ using full information algorithm $\mathcal{A}$ with
loss $\widetilde{\ell}^t$ for round $t$.
\ENDFOR
\end{algorithmic}
\end{algorithm}
Lemma \ref{lem:bounding_total_freezing} bounds the probability that is frozen via this process. Since we select an arm based on the distribution normalized on the non-frozen arms (Step~\ref{step_normalize_black_box} in Algorithm~\ref{alg:black_box}), this probability multiplies the loss of our algorithm in the final regret, making our bounds relate to the loss of the best rather than the worst arm.

For clarity of presentation we first provide the 
approximate regret guarantee in expectation (Theorem \ref{thm:black_box}) and then show its high-probability version (Theorem \ref{thm:black_box_whp}), in both cases assuming that the algorithm has access to an upper bound of the maximum independence number $\alpha$  as an input parameter. 
In Theorem \ref{thm:black_box_small_loss} we provide the small-loss version of the above bound without explicit knowledge of this quantity.

\vspace{0.1in}
\begin{lemma}\label{lem:bounding_total_freezing}
 At every round $t$, the total probability of frozen arms is at most $\eps'$:
$\sum_{i \in F^t}^tp_i^t \le \epsilon'$, and hence
any non-frozen arm $i$ increases its probability due to freezing by a factor of at most $(1-\eps')$.
\end{lemma}
\begin{proof} We first consider the arms that are frozen due to the $\gamma$-threshold (line 3 of the algorithm). Claim \ref{Claim_3gamma_frozen} below shows that the total probability frozen in the initial step is bounded by $\sum_{i\in F_0^t} p_i^t\leq \alpha(G^t)\gamma$. We then focus on the arms frozen due to the recursive $\gamma'$-threshold (line 4 of the algorithm). Claim \ref{Claim_2} below bounds the total probability frozen in the propagation process by three times the total probability frozen in the initial step. Combining the two claims, we obtain: 
\begin{equation*}
    \sum_{i\in F^t}p_i^t=
    \sum_{i\in F_0^t}p_i^t+\sum_{i\in \bigcup_{k\geq 1} F_k^t}p_i^t\leq  
    \alpha(G^t)\gamma+3\alpha(G^t)\gamma=4\alpha(G^t)\gamma \leq\eps'.
\end{equation*}
The lemma then follows from the relation in the normalization step of the algorithm (line 5). 
\end{proof}

\begin{claim}\label{Claim_3gamma_frozen} 
The total probability frozen in the initial step is bounded by $\sum_{i\in F_0^t} p_i^t\leq \alpha(G^t) \gamma$. \vspace{-0.1in}
\end{claim}
\vspace{0.05in}
\begin{proof}
Let $S^t$ be a maximal independent set on $F_0^t$. Since the independent set is maximal, every node in $F_0^t$ either is in $S^t$ or has a neighbor in $S^t$, so we obtain:
\begin{equation*}\label{eq:3gamma_frozen}
\sum_{i\in F_0^t} p_i^t\leq \sum_{i\in S^t}\sum_{j \in \prn*{N_i^t\cap  F_0^t}}p_j^t < \alpha(G^t) \cdot \gamma.
\end{equation*}
where the last inequality follows from the fact that 
there are at most $\alpha(G^t)$ nodes in $S^t$ and, since they are frozen, the probability of being observed is at most $\gamma$ for each of them.
\end{proof}

\begin{claim} \label{Claim_2} The total probability frozen in the propagation steps is bounded by three times the total probability frozen at the initial step. More formally:
\begin{equation*}
\sum_{i \in \bigcup_{k\geq 1} F_k^t}p_i^t\leq 3\sum_{i\in F_0^t} p_i^t. 
\end{equation*}
\end{claim}
\begin{proof}
The purpose of the lower threshold $\gamma'$ in line 4 is to limit the propagation of frozen probability. Consider an arm $i$ frozen on step $k\geq 1$. Since arm $i$ was not frozen at step $0$, the initial probability of being observed by any node of $G^t$ is at least $\gamma=3\gamma'$. When this arm becomes frozen, it is observed with probability at most $\gamma'$. Hence $2\gamma'$ of the original probability stems from arms frozen earlier. Using this, we can bound the probability mass in $F_1^t$ by at most 1.5 times the mass of $F_0^t$. Further, from these arms at most $\gamma'$ of the originally at least $3\gamma'$ probability is newly frozen, and hence can affect non yet frozen arms, creating a further cascade. We show that the total frozen probability can be at most $3$ times the probability of nodes in $F_0^t$. The proof of this fact follows in a way that is analogous of how the number of internal nodes of a binary tree is bounded by the number of leaves, as any node can have at most 1 parent, while having 2 children.

More formally, we consider an auxiliary function that serves as an upper bound of the left hand side and a lower bound of the right hand side, proving the claim. The claim is focused on a single round $t$. For simplicity of notation, we drop the dependence on $t$ from the notations, i.e., use $F=\cup_k F_k$ for the set of nodes frozen, $p_i$ for the probability of node $i$, use $G$ for the graph, and $E$ for its edge-set. Let $F_{\ge 1}=\bigcup_{k\geq 1} F_k^t$. We order all nodes in $F$ based on when they are frozen. More formally, if $i\in F_m$ and $j\in F_k$ with $m<k$ then $i\prec j$. This is a partial ordering as $\prec$ does not order nodes frozen at the same iteration of the recursive freezing. We now introduce the heart of the auxiliary function which lies in the sum of the products of probabilities  $p_i \cdot p_j$ along  edges $(i,j)$ with $i\prec j$, such that  $(i,j)\in E$, i.e.
\begin{equation}\label{eq:potential_graphs}
\sum_{\substack{i\in F,j\in F_{\ge 1},i\prec j\\(i,j)\in E}}p_i p_j
\end{equation}
To lower bound the quantity in \eqref{eq:potential_graphs}, we sum over $j$ first. Node $j$ was not in $F_0$ so its neighborhood has a total probability mass of at least $\gamma=3\gamma'$. By the time $j$ is frozen, the remaining probability mass is less than $\gamma'$, so a total probability mass of at least $2\gamma'$ must come from earlier frozen neighbors.
$$
\sum_{\substack{i\in F, j \in F_{\ge 1}, i\prec j\\ (i,j)\in E}} p_i p_j =
\sum_{j\in F_{\geq 1}} p_j \brk*{\sum_{\substack{i\in F, i\prec j\\ (i,j)\in E}} p_i }\ge
\sum_{j\in F_{\ge 1}} p_j \cdot 2\gamma'
$$
To upper bound the quantity in \eqref{eq:potential_graphs}, we sum over $i$ first, and separate the sum for $i \in F_0$ and $i \in F_{\geq 1}$. Nodes $i\in F_0$ have a total probability of less than $\gamma=3\gamma'$ in their neighborhood, as they are frozen in line 3 of the algorithm. Nodes $i \in F_{\ge 1}$ have at most $\gamma'$ probability mass left in their neighborhood when they become frozen and therefore contributes at most this much total probability to the products with neighbors later in the ordering.
$$
\sum_{\substack{i\in F, j \in F_{\ge 1}, i\prec j\\ (i,j)\in E}} p_i p_j =
\sum_{i\in F_0}  p_i \brk*{\sum_{\substack{j\in F_{\ge 1}, i\prec j\\(i,j)\in E}} p_j }+
\sum_{i \in F_{\ge 1}} p_i \brk*{\sum_{\substack{j \in F_{\ge 1}, i\prec j\\(i,j)\in E}} p_j} \le 
\sum_{i \in F_0} p_i \cdot
3\gamma'+ \sum_{i \in F_{\ge 1}} p_i\cdot  \gamma'
$$
The above lower and upper bounds imply that
$3\gamma'\sum_{i \in F_0}p_i  +\gamma'\sum_{i \in F_{\ge 1}} p_i  \ge 2\gamma'\sum_{i\in F_{\ge 1}} p_i $
and hence we obtain the claimed bound (reintroducing the round $t$ in the notation):
\begin{equation*}\label{eq:gamma_frozen}
\sum_{i \in \bigcup_{k\geq 1} F_k^t} p_i^t \le 3\sum_{i \in F_0^t} p_i^t.
\end{equation*}
\end{proof}

\textbf{Bounding pseudoregret.}
We are now ready to prove our first result: a bound for learning with partial information based on feedback graphs. We first provide the guarantee for approximate pseudoregret in expectation. We assume both the learning rate $\eps$ as well as an upper bound  $\alpha$ on the size of the independent sets are given as an input. At the end of this section, we show how the results can be turned into regret guarantees via doubling trick without knowledge of the independence number. 

\begin{theorem}
\label{thm:black_box}
\vspace{0.1in} Let $\mathcal{A}$ be any full information algorithm with an expected approximate regret guarantee given by: $\En\brk*{\apx\prn*{f,\eps/2}} \le \frac{2 L\cdot A(d,T)}{\eps}$  against any arm $f$, when run on losses in $[0,L]$. The \emph{Double-Threshold Freezing Algorithm} (Algorithm~\ref{alg:black_box}) run with learning parameter $\epsilon'=\frac{\eps}{2}$ on input $\mathcal{A}$, $\alpha$, $d$, has expected $\epsilon$-approximate regret guarantee: $\En\brk*{\apx\prn*{f,\eps}}=48\alpha\cdot A(d,T)/\eps^2$. 
\end{theorem}
\begin{proof}
The proof carefully applies freezing to overcome the shortcomings in the classical reductions to bandit feedback described in the beginning of the section. First, freezing guarantees that the maximum estimated loss is $L=1/\gamma'$ (since the probability of being observed is at least $\gamma'$ for any non-frozen arm; else this arm becomes frozen at step 4 of the algorithm). Second, although the estimator is no longer unbiased for all arms, it is unbiased for all non-frozen arms $i\notin F^t$ at all rounds $t$, i.e., $\sum_{j\in N_i^t}w_j^t\widetilde{\ell}_j^t=\ell_i^t$. Third, it always negatively (optimistically) biased regardless whether the arm is frozen or not, i.e., $\En\brk*{\widetilde{\ell}_i^t}\leq \ell_i^t$. Finally, the frozen probability is distributed to all the non-frozen arms proportionally to their probabilities prior to the normalization. This is in contrast to mixing a uniform action distribution where the extra probability is distributed across all arms uniformly, resulting to guarantees that scale with the performance of the worst arm. Formally: 
\begin{align*}
    (1-\eps)\En\brk*{\sum_t \ell^t_{I(t)}} &=(1-\eps)\En\brk*{\sum_t \sum_i w_i^t \widetilde{\ell_i^t}} && \text{as $I(t)\notin F^t$ and $\widetilde{\ell^t}$ is unbiased on all $i\notin F^t$.}\\
    &\leq \frac{1-\eps}{1-\frac{\epsilon}{2}}\En\brk*{\sum_t \sum_i p_i^t \widetilde{\ell_i^t}} &&\text{by Lemma \ref{lem:bounding_total_freezing}.}\\
    &\leq \En\brk*{{\sum_t \widetilde{\ell_f^t}}+L\cdot \frac{A(d,T)}{\eps'}} && \text{by the low approx regret of $\mathcal{A}$.}\\
    &\leq \sum_t \En\brk*{\ell_f^t}+\frac{A(d,T)}{
    \gamma'\cdot  \eps'}&&\text{as the estimator is negatively biased}\\
    &= \sum_t \En\brk*{\ell_f^t}+48\alpha\cdot\frac{A(d,T)}{\eps^2} && \text{using definitions of $L$, $\gamma'$, $\gamma$ and $\eps'$.}
\end{align*}
The second inequality also uses the fact that $(1-\epsilon)\leq (1-\epsilon/2)^2$.
\end{proof}
Notice that it was important to be able to use a freezing threshold $\gamma\propto \eps/\alpha$ instead of $\gamma\propto \eps/d$ for the above analysis, allowing an approximate regret bound with no dependence on $d$.

We now instantiate the theorem with a particular algorithm. Multiplicative Weights satisfies the condition in the theorem with $A(d,T)=\frac{\log(d)}{2}$ which automatically implies the following corollary. 
\begin{corollary}
The \emph{Double-Threshold Freezing Algorithm} (Algorithm~\ref{alg:black_box}) run with Multiplicative Weights as the full information subroutine with learning parameter $\epsilon'=\frac{\eps}{2}$ on input $\mathcal{A}$, $\alpha$, $d$, has expected $\epsilon$-approximate regret guarantee: $\En\brk*{\apx\prn*{f,\eps}}\leq 24\alpha\cdot \frac{\log d}{\eps^2}
$.
\end{corollary}

\textbf{High probability bound.}
To obtain a high-probability guarantee (and hence a bound on the regret, not only the pseudoregret), we encounter an additional complication since we need to upper bound the cumulative estimated loss of the comparator by its cumulative actual loss. For this purpose, the mere fact that the estimator is negatively biased does not suffice.
The estimator may, in principle, be unbiased (if the arm is never frozen), and the variance it suffers can be
high, which could ruin the small-loss guarantee. To deal with this, we apply a concentration inequality, comparing the expected loss to a multiplicative approximation of the actual loss. This is inspired by the approximate regret notion, is a quantity with negative mean, and has variance that depends on $1/\eps$ as well as the magnitude of the estimated losses which is $\frac{1}{\gamma'}$. 
\begin{theorem}
\label{thm:black_box_whp}
Let $\mathcal{A}$ be any full information algorithm with an expected approximate regret guarantee of: $\En\brk*{\apx\prn*{f,\eps/5}} \le \frac{5 L\cdot A(d,T)}{\eps}$,  against any arm $f$, when run on losses in $[0,L]$. For any $\delta>0$, with probability $1-\delta$, the \emph{Double-Threshold Freezing Algorithm} (Algorithm~\ref{alg:black_box}) run with learning parameter $\epsilon'=\frac{\eps}{5}$ on input $\mathcal{A}$, $\alpha$, $d$, has $\epsilon$-approximate regret of: $$\apx\prn*{f,\eps}\leq\frac{100\alpha\cdot \prn*{A\prn*{d,T}+3\log\prn*{(d+2)/\delta}}}{\eps^2}.
$$
\end{theorem}
To prove this, we need the following concentration inequality, showing that the sum of a sequence of (possibly dependent) random variables cannot be much higher than the sum of their expectations: 
\begin{lemma}\label{lem:black_box_high_probability}
Let $x_1,x_2,\dots,x_T$
be a sequence of non-negative random variables, s.t. $x_t \in [0,1]$. Let  $\En_{t-1}\brk*{x_t}=\En\brk*{x_t|x_1, \ldots, x_{t-1}}$.  Then, for any $\epsilon, \delta > 0$,  with probability at least $1-\delta$
$$
\sum_{t=1}^T x_t -\prn*{1+\eps} \sum_{t=1}^T \mathbb{E}_{t-1}[x_t] \leq \frac{\prn*{1+\eps}\ln\prn*{1/\delta}}{\eps} $$
and also with probability at least $1-\delta$
$$
\prn*{1-\eps}\sum_{t=1}^T \mathbb{E}_{t-1}[x_t] -
\sum_{t=1}^T x_t \leq \frac{\prn*{1+\eps}\ln\prn*{1/\delta}}{\eps}
$$
\end{lemma}
The proof follows the outline of classical Chernoff bounds for independent variables combined with the law of total expectation to handle the dependence. For completeness, the proof details are provided in Appendix \ref{app:proofs_section_3}.
\begin{proof}[Proof of Theorem \ref{thm:black_box_whp}]
To obtain a high-probability statement, we use Lemma \ref{lem:black_box_high_probability} multiple times as follows:
\begin{enumerate}
\setlength{\itemsep}{0pt}\setlength{\parsep}{0pt}\setlength{\parskip}{0pt}
    \item \label{alg-loss} Show that the sum of the algorithm's losses stays close to the sum of the expected losses.
  \item \label{alg-fullInfo} Show that the sum of the expected losses stays close to the sum of the expected estimated losses used by the full information algorithm $\mathcal{A}$ 
  \item \label{comparator} Show that the sum of the estimated losses of each arm $f$  stays close to the sum of the actual losses.
\end{enumerate}
Starting with the item \ref{alg-loss}, we use $x_t=\ell^t_{I(t)}$, and note that its expectation conditioned on the previous losses is $m_t=\sum_i w^t_i \ell^t_i$ so we obtain that, for any $\delta' , \eps >0$, with probability at least $(1-\delta')$
$$
\sum_t \ell^t_{I(t)}-(1+\eps')\sum_t \sum_i w^t_i \ell^t_i \le \frac{(1+\eps')\ln(1/\delta')}{\eps'}
$$

Next item \ref{comparator}, for a comparator $f$ we use the lemma with $x_t=\widetilde \ell^t_f$ and its expectation $m_t=\ell^t_f$. Now $x_t$ is bounded by $\frac{1}{\gamma}$ and not 1, so by scaling we obtain that with probability $(1-\delta')$
$$
\sum_t \widetilde \ell^t_f-(1+\eps')\sum_t \ell^t_f \le \frac{(1+\eps')\ln(1/\delta')}{\gamma \eps'}
$$
Finally, we use the lower bound in the lemma to show item \ref{alg-fullInfo}: for $x_t=\sum_i p^t_i \widetilde \ell^t_i$, the expected losses observed by the full information algorithm, and its expectation $m_t=\sum_i p^t_i \ell^t_i$. Again, the $x_t\in [0,\frac{1}{\gamma}]$ so we obtain that with probability $(1-\delta')$, 
$$
\sum_t \sum_i p^t_i \ell^t_i-(1+\eps')\sum_t \sum_i p^t_i \widetilde \ell^t_i \le \frac{(1+\eps')\ln(1/\delta')}{\gamma \eps'}
$$
Using union bound and $\delta'=\delta/(d+2)$, all these inequalities hold simultaneously for all $\delta>0$. To simplify notation, we use $B=\frac{(1+\eps')\ln((d+2)/\delta)}{\gamma \eps'}$ for the error bounds above.

Combining all the bounds we obtain that
\begin{align*}
    \sum_t \ell^t_{I(t)} &\le (1+\eps')\sum_t \sum_i w^t_i \ell^t_i +B &&\text{by item \ref{alg-loss} above}\\
    &\le \frac{1+\eps'}{1-\eps'}\prn*{\sum_t \sum_i p^t_i \ell^t_i +B} && \text{by Lemma \ref{lem:bounding_total_freezing}}\\
    &\le \frac{(1+\eps')^2}{1-\eps'}\prn*{\sum_t \sum_i p^t_i \widetilde \ell^t_i +2B}&&\text{by item \ref{alg-fullInfo} above}\\
    & \le \frac{(1+\eps')^2}{(1-\eps')^2}\prn*{ \sum_t \widetilde \ell_f^t+2B+\frac{A(d,T)}{\gamma\cdot  \eps'}} && \text{by the low approx. regret of $\mathcal{A}$}\\
    &\le   \frac{(1+\eps')^3}{(1-\eps')^2}\prn*{ \sum_t \ell_f^t+3B+\frac{A(d,T)}{\gamma\cdot  \eps'}} && \text{by \ref{comparator} applied to $f$}
\end{align*}
The theorem then follows as $\frac{(1+\eps')^3}{(1-\eps')^2}\leq (1-\eps)^{-1}$ for $\eps'=\frac{\eps}{5}$.
\end{proof}

\textbf{The small-loss guarantee without knowing $\alpha$.} 
We presented the results so far in terms of approximate regret and assuming we have $\alpha$, an upper bound for the maximum independent set, as an input. Next we show that we can use this algorithm with the classical doubling trick without knowing $\alpha$, and achieving low regret both in expectation as well as with high probability, not only approximate regret. We start with a large $\epsilon$ and small $\alpha$ and halve and double them respectively, when observing that they are not set right. There are two issues worth mentioning. 

First, unlike full information, partial information does not provide access to the loss of the comparator $L^{\star}$. As a result, we apply doubling trick on the loss of the algorithm instead and then bound the regret of the algorithm appropriately. This is formalized in the following lemma which follows standard doubling arguments and whose proof is provided in Appendix \ref{app:proofs_section_3} for completeness.
\begin{lemma}[standard doubling trick]\label{lem:double}
Suppose we have a randomized algorithm that takes as input any $\eps > 0$ and guarantees that, for some $q\geq 1$ and some function $\Psi(\cdot)$, and any $\delta>0$, with probability $1-\delta$, for any time horizon $s$ and any comparator $f$: 
$$
(1 - \eps)\sum_{t=1}^s \ell^t_{I(t)} \le \sum_{t=1}^s \ell^t_{f} + \frac{\Psi(\delta)}{\epsilon^q}.
$$
Assume that we use this algorithm over multiple phases (by restarting the algorithm when a phase ends): for some $\kappa>1$, we run each phase $\tau$ with $\epsilon_\tau = \kappa^{-\tau}$ until
$\epsilon_\tau \widehat{L}_\tau > \frac{\Psi(\delta)}{(\eps_{\tau})^q}$
where $\widehat{L}_{\tau}$ denotes the cumulative loss of the algorithm for phase $\tau$. Then, for any $\delta >0$, the regret for this multi-phase algorithm is bounded, with probability at least $1-\delta$ as:
\begin{align*}
    \mathrm{Reg}\leq
    \prn*{\widetilde{\Psi}(\delta,L^{\star})}^{\frac{1}{q+1}}\cdot
    \prn*{(q+1)\prn*{L^{\star}+\frac{\kappa}{\kappa-1}}+\prn*{\widetilde{\Psi}(\delta,L^{\star})}}^{\frac{q}{q+1}}+\frac{\kappa}{\kappa-1}.
\end{align*}
where $\widetilde{\Psi}(\delta,L^{\star})=\frac{4\kappa^{2q^2+2q+1}+1}{(\kappa-1)^{q+1}}\cdot \Psi(\frac{\delta}{\log(L^{\star}+1)+1})$.
Treating all terms other than $L^{\star}, T, \delta$ as constants:
$$\mathrm{Reg} =.  O\left(\prn*{L^{\star}}^\frac{q}{q+1}\Psi\prn*{\frac{\delta}{\log(L^\star+1)+1}}^{\frac{1}{q+1}}+\Psi\prn*{\frac{\delta}{\log(L^{\star}+1)+1}}\right).$$
\end{lemma}
Second, observing the maximum independent set is challenging since this task is NP-hard to approximate. However, if one looks carefully into our proofs, we just require knowledge of a maximal independent set on the $\gamma$-frozen arms and not one of maximum size. This can be easily computed greedily at each round and therefore our algorithm can handle changing graphs without requiring knowledge of the maximum independence number. Combining these two observations, we prove the following small-loss guarantee. 

\begin{theorem} \label{thm:black_box_small_loss}
Let $\mathcal{A}$ be any full information algorithm with $\eps$-approximate regret bounded by $\frac{L\cdot A(d,T)}{\eps}$ when run on losses in $[0,L]$ and with parameter $\eps>0$. If one runs the \emph{Dual-Threshold Freezing Algorithm} (Algorithm~\ref{alg:black_box}) as in Theorem~\ref{thm:black_box_whp} and using the doubling scheme as in Lemma \ref{lem:double} with $\kappa=1.2$ and tuning $\alpha$ appropriately on each phase, then for any $\delta >0$, with probability at least $(1-\delta)$ the regret of this algorithm is bounded by 
$\bigO\prn*{\prn*{\prn*{L^\star}^{\frac{2}{3}}
\prn*{\alpha 
A(d,T)}^{\frac{1}{3}}+\alpha A(d,T)}\log\prn*{\frac{d \log (L^{\star}+1)}{\delta}}\log(\alpha)}$. Expressing this in a non-asymptotic form, letting $\bar{\Psi}=200+600\cdot\log\prn*{\frac{d(d+2)(\log(L^{\star}+1)+1)\log(\alpha)}{\delta}}$:
\begin{eqnarray*}\mathrm{Reg} \leq \prn*{\prn*{3(L^{\star}+11)+\alpha A(d,T)\cdot\bar{\Psi}}^{\frac{2}{3}}\cdot \prn*{\alpha A(d,T)\cdot\bar{\Psi}}^{\frac{1}{3}}+12}\log(\alpha).
\end{eqnarray*}
\end{theorem}
\begin{proof}
First for simplicity assume that $\alpha$ is known in advance. In this case, using Theorem \ref{thm:black_box_whp}, we can conclude that for any $\delta, \eps >0$, Algorithm \ref{alg:black_box} run with $\mathcal{A}$ enjoys an $\epsilon$-approximate regret guarantee of $\frac{\Psi(\delta)}{\eps^q}$ for $\Psi(\delta)=100\alpha\cdot \prn*{A\prn*{d,T}+3\log\prn*{(d+2)/\delta}}$ and $q=2$. This means that Lemma~\ref{lem:double} can be applied with $\widetilde{\Psi}(\delta,L^{\star})\leq 20\Psi(\frac{\delta}{\log(L^{\star}+1)+1})$. Hence, running Algorithm \ref{alg:black_box} while tuning $\eps$-parameter using doubling trick as in Lemma \ref{lem:double} yields a regret guarantee of $$
    \prn*{3(L^{\star}+11)+\alpha A(d,T)\cdot\widehat{\Psi}}^{\frac{2}{3}}\cdot \prn*{\alpha A(d,T)\cdot\widehat{\Psi}}^{\frac{1}{3}}+11,$$
where $\widehat{\Psi}=200+600\log\prn*{\frac{d(d+2)(\log(L^{\star}+1)+1)}{\delta}}$.

If $\alpha$ is not known in advance, we can begin with a guess (say $\alpha'=1$) and double the guess every time that this is incorrect, i.e. the maximal independent set of the $\gamma$-frozen nodes has more than $\alpha'$ nodes. We make at most $\log\prn*{\alpha}$ updates. Within one phase with the same update, the previous guarantee holds with probability at least some $\delta'$. At the time of the update we can lose an extra of at most $1$. For the rest of the rounds, the guarantees work additively.  Therefore, setting $\delta'=\frac{\delta}{\log(\alpha)}$, we obtain the previous guarantee with an extra $\log(\alpha)$ decay in the guarantee. Since $\alpha<d$, the dependence on $\log(\alpha)$ is dropped in the $\bigO$ notation of the regret bound.
\end{proof}

\begin{remark}
The above approach uses doubling trick to obtain the guarantees without assuming any parametric form for the full-information algorithm. However, if the full-information algorithm is multiplicative weights or other mirror descent algorithms with a notion of step-size, the above guarantees can also be obtained via appropriate adaptive step-sizes circumventing the need for doubling trick.
\end{remark}
In the next two sections, we consider applications of this general technique to various special cases, and show that the  ${(L^\star)}^{2/3}$ dependence on the loss can be improved to $\sqrt{L^\star}$ in many applications when combined with appropriately selected full-information algorithms. It is an interesting open problem if this can be achieved for the general graph-based feedback discussed in this section.

\textbf{On the necessity of double thresholding.} After the initial appearance of this work, Liu and Sellke observed that the use of two thresholds is not essential for the above result if one appropriately modifies Claim~\ref{Claim_3gamma_frozen}  \cite{LiuSellke}. In particular, one can initially select any node with probability of observation less than $\gamma$ (call this node as \emph{special}) and freeze both this node and all its neighbors regardless whether the probability they are observed is less than $\gamma$. This process continues with selecting additional special nodes that are now observed with probability than than $\gamma$, freezing them and all their neighbors, etc., until no node has probability of observation less than $\gamma$. Note that all special nodes form an independent set (as all neighbors of a special nodes are frozen with it). Also the probability frozen associated to any special node is at most $\gamma$. As a result, the total probability frozen is at most $\gamma\cdot \alpha$, controlling the subsequent freezing cascade without two thresholds.

\section{High-probability $\sqrt{L^{\star}}$ bounds for pure bandits and fixed graphs}\label{sec:optimal_lst_highprob}
In this section, we show how combining the aforementioned black-box framework with properties of particular full-information algorithms leads to optimal dependence on the loss of the best arm $L^{\star}$ for important special feedback settings. In particular, in Section~\ref{ssec:optimal_bandit}, we present an optimal high-probability $\widetilde{\bigO}(\sqrt{L^{\star}})$ bound for the classical bandit feedback setting, answering an open question of Neu \cite{Neu15_semibandits} using exponential weights as the full-information algorithm. In Section~\ref{ssec:clique_partition}, we extend this guarantee to the graph-based feedback setting for the case where the graph but restrict our attention on fixed feedback graphs --- for this case we provide an optimal dependence on $L^{\star}$ at the expense of a dependence on the minimum clique partition instead of the independence number; as we discuss in Section~\ref{sec:conclusions}, going beyond fixed graphs and clique partition is probably the most interesting open direction from our work. Finally, in the next section (Section~\ref{ssec:semi_bandits_black_box}), we show how using a variant of follow the perturbed leader as full-feedback algorithm can lead to optimal high-probability guarantees for a different feedback setting, combinatorial semi-bandits, answering open questions raised in \cite{Neu15_semibandits,Neu2015_implicit}.

\subsection{Multi-armed bandits}\label{ssec:optimal_bandit}
To achieve optimal dependence on $L^{\star}$, we need to better understand the places where the inefficiency arises. The first such place is when we apply the bound of the full-information algorithm which, in a black-box analysis, needs to have dependence both on the magnitude of losses, $L \le 1/\gamma'$, and on the approximation
parameter $\eps'$. Instead of applying this bound, we provide a refined analysis that relates the expected estimated loss of the full information algorithm to the sum of the cumulative estimated losses of all the arms. Using exponential weights as a full-information algorithm guarantees that the cumulative estimated losses of all the arms are close to each other (Lemma \ref{lem:estimated_losses_close}) which enables us to remove this 
inefficiency. This was also used 
by Allenberg et al. \cite{Allenberg2006} to prove optimal pseudo-regret guarantees but their analysis did not extend to high-probability. To derive the high-probability guarantee, we address the second inefficiency of the black-box, where to bound the negative bias of the comparator's cumulative estimated loss by its cumulative actual loss, we again had dependence on both the magnitude of the estimated losses and $\eps$. For that we apply the implicit exploration idea of Koc\'ak et al. \cite{Kock2014EfficientLB} which creates a negative bias to all arms and not only the arms that are frozen (Lemma \ref{lem:implicit_exploration}). Although Neu \cite{Neu2015_implicit} used implicit exploration to provide high-probability uniform bounds his results did not extend to small-loss. Combining our framework with both exponential weights and implicit exploration, we obtain an algorithm we term \emph{GREEN-IX} (Algorithm \ref{alg:green_ix}) that, with  high-probability, guarantees regret bound of 
$\bigO(\sqrt{L^{\star}})$. 
\begin{theorem}\label{thm:green_ix} GREEN-IX run with learning parameter $\epsilon'=\frac{\epsilon}{2}$ guarantees an $\epsilon$-approximate regret of $\frac{6d\log(d^2/\delta)}{\eps} +d\prn*{1+2\ln(2d/\eps)+\log(d^2/\delta)}$
    with probability at least $1-\delta$.
\end{theorem}
Combining the above approximate regret guarantee with the doubling trick described in the previous section, we obtain a small-loss bound for GREEN-IX as a corollary; the proof follows similarly to the one of Theorem \ref{thm:black_box_whp} by applying Lemma \ref{lem:double} with $\Psi(\delta)=\bigO\prn*{d\log(d/\delta)}$ and $q=1$.
\begin{corollary}\label{cor:green_ix}
GREEN-IX with doubling trick on $\eps$ has regret $\widetilde{\bigO}\prn*{\sqrt{d\log(d/\delta)\cdot L^{\star}}+d\log(d/\delta)}$ with probability at least $1-\delta$, and hence expected regret at most $\widetilde{\bigO}(\sqrt{d\log(d)\cdot L^{\star}}+d\log(d))$.
\end{corollary}

\begin{algorithm}[!h]
\caption{GREEN-IX}
\label{alg:green_ix}
\begin{algorithmic}[1]
\REQUIRE Number of arms $d$, learning parameter $\epsilon'$, freezing threshold $\gamma:=\frac{\epsilon'}{d}$, full-feedback learning rate $\eta:=\frac{\epsilon'}{2d}$, implicit exploration parameter $\zeta:=\frac{\epsilon'}{2d}$.
\STATE Initialize $p_i^1$ for arm $i$ uniformly ($p_i^t=\frac{1}{d}$) and set $t=1$ (round 1).
\FOR{$t=1$ \TO $T$}
\STATE Freeze arm $i$ if its probability $p_i^t$ is below freezing threshold $\gamma$ to create the set $
F^t= \crl*{i: p_i^t< \gamma}.
$
\STATE Normalize the probabilities of non-frozen arms so that they form a distribution.\vspace{-0.1in}
$$
w_i^t=\begin{cases}
0 & \text{ if } i \in 
F^t\\
\frac{p_i^t}{1-\sum_{j\in F^t
} p_j^t} & \text{ else }
\end{cases}
$$\vspace{-0.18in}
\STATE Draw arm $I(t)\sim w^t$ and incur loss $\ell_{I(t)}^t$.
\STATE  Compute biased estimate of losses via implicit exploration with parameter $\zeta$:\vspace{-0.1in}
$$
\widetilde{\ell}_i^t=\begin{cases}\frac{\ell_i^t}{w_i^t+\zeta} & \text{ if } i=I(t) \\
0 & \text{ else }
\end{cases}.
$$\vspace{-0.2in}
\STATE Update $p_i^{t+1}$ via exponential weights update with learning rate $\eta$:~~ $p_i^{t+1}\propto p_i^t \exp(-\eta \widetilde{\ell_i^t})$. 
\ENDFOR
\end{algorithmic}
\end{algorithm}

\begin{lemma}[implied by proof of Theorem 2 in \cite{Allenberg2006}]\label{lem:estimated_losses_close}
GREEN-IX satisfies
for any arms $i,j$:
$$
\sum_{t=1}^T \widetilde{\ell}_i^t\leq \sum_{t=1}^T \widetilde{\ell}_j^t +\frac{1}{\gamma}+\frac{\ln(1/\gamma)}{\eta}
$$
\end{lemma}
\begin{proof}
The proof follows the lines of Theorem 2 in \cite{Allenberg2006} adding a $1/\gamma$ term that was missing in their analysis. Let $T_i$ be the last round that $i$ is not frozen. Thus its probability is then greater than $\gamma$.
$$
\gamma\leq p_i^{T_i}=
\frac{\exp\prn*{-\eta\sum_{t=1}^{T_i-1}\widetilde{\ell}_i^t}}{\sum_k \exp\prn*{ -\eta\sum_{t=1}^{T_i-1}\widetilde\ell_k^t}}
\leq \frac{\exp\prn*{-\eta\sum_{t=1}^{T_i-1}\widetilde{\ell}_i^t}}{ \exp\prn*{- \eta\sum_{t=1}^{T_i-1}\widetilde{\ell}_j^t}}
$$
As a result:
$$
\sum_{t=1}^{T_i-1}\widetilde{\ell_i^t}\leq \sum_{t=1}^{T_i-1}\widetilde{\ell_j^t} +\frac{\ln(1/\gamma)}{\eta}\Rightarrow \sum_{t=1}^T \widetilde{\ell_i^t}\leq \sum_{t=1}^T \widetilde{\ell_j^t}+\frac{1}{\gamma}+\frac{\ln(1/\gamma)}{\eta}, 
$$
where the last inequality follows as $\widetilde{ \ell_i^t} \le 1/\gamma$ for all arms at all times and the estimated loss of $i$ is $0$ after round $T_i$ by definition of $T_i$.
\end{proof}

\begin{lemma}[implied by Corollary 1 in \cite{Neu2015_implicit}]\label{lem:implicit_exploration} 
With probability at least $1-\delta$, any full information algorithm run on estimated losses $\widetilde{\ell^t}$ with implicit exploration satisfies for all arms $i\in[d]$ simultaneously:
$$
\sum_{t=1}^T \prn*{\widetilde{\ell_i^t}-\ell_i^t}\leq \frac{\log(d/\delta)}{2\zeta} 
$$
\end{lemma}
\begin{proof}
Let $\bar{\ell}_i^t$ be fictitious losses that are equal to the actual losses for arms $i\notin F^t$ and $0$ for $i\in~F^t$. Corollary 1 in \cite{Neu2015_implicit} establishes that: $\sum_{t=1}^T \prn*{\widetilde{\ell}_i^t -\bar{\ell}_i^t}\leq \frac{\log(d/\delta)}{2\zeta}$ simultaneously for all $i$ with probability at least $1-\delta$. The lemma follows noting that $ \bar \ell_i^t\le \ell_i^t$ as the estimator has negative bias.
\end{proof}

\begin{lemma}[following from the analysis in Theorem 2.4 of \cite{prediction_book} or Theorem 3.1 of \cite{BubeckC12}]\label{lem:mwu_second_order_regret}
Exponential weights with learning rate $\eta$ applied on the estimated losses satisfies:
$$
\sum_t\sum_i p_i^t\widetilde{\ell_i^t}-\sum_t \widetilde{\ell_f^t} \leq \eta \sum_t \sum_i p_i^t \prn*{\widetilde{\ell_i^t}}^2+\frac{\log(d)}{\eta}
$$
\end{lemma}
\begin{proof}[Proof of Theorem \ref{thm:green_ix}]
The proof follows the roadmap of the proof of Theorem \ref{thm:black_box_whp} but handles the suboptimal places of the black-box theorem's proof by applying Lemmas \ref{lem:estimated_losses_close} and \ref{lem:implicit_exploration}. We show that for each arm $f$, the guarantee holds with failure probability $\delta'=\delta/d$. Therefore the guarantee holds against all the arms $f$ simultaneously with probability at least $1-\delta$. More formally: 
\begin{align*}
    (1-\eps)\sum_t  \ell^t_{I(t)} &=(1-\eps)\sum_t \sum_i\prn*{w_i^t+\zeta}\cdot \widetilde{\ell_i^t} &&\text{by definition of $\widetilde{\ell_i^t}$}\\
    &\leq\frac{1-\eps}{1-\eps'}\sum_t \sum_i p_i^t \widetilde{\ell_i^t}+  \zeta\sum_t \sum_i \widetilde{\ell_i^t} &&\text{by Lemma \ref{lem:bounding_total_freezing}}\\
    &\leq \frac{1-\eps}{1-\eps'
    }\sum_t \widetilde{\ell_f^t} + \eta \sum_t \sum_i p_i^t \prn*{\widetilde{\ell_i^t}}^2+\frac{\log(d)}{\eta}+\zeta\sum_t \sum_i \widetilde{\ell_i^t} && \text{by Lemma  \ref{lem:mwu_second_order_regret}}\\
    &\leq \frac{1-\eps}{1-\eps'
    }\sum_t \widetilde{\ell_f^t} +(\eta+\zeta)\sum_t \sum_i\widetilde{\ell_i^t}+\frac{\log(d)}{\eta} && \text{as $\ell_i^t \le 1$ and $p_i^t\le w_i^t+\zeta$}\\
    &\leq \frac{1-\eps}{1-\eps'    }
    \sum_t \widetilde{\ell_f^t}+(\eta+\zeta)\sum_{t=1} d \widetilde{\ell}_f^t\\
    &\qquad + d(\eta+\zeta)\prn*{\frac{1}{\gamma}+\frac{\ln(1/\gamma)}{\eta}}+\frac{\log(d)}{\eta}&& \text{by Lemma \ref{lem:estimated_losses_close}}
\end{align*}
Now we use the strict negative bias of Lemma \ref{lem:implicit_exploration} to get that with probability at least $(1-\delta')$ we can continue the above inequalities as:
\begin{align*}
  (1-\eps)\sum_t  \ell^t_{I(t)} &\leq 
    \frac{1-\eps}{1-\eps'}\sum_t \ell_f^t+\frac{\log(d/\delta')}{2\zeta}+(\eta+\zeta)\sum_{t=1} d \ell_f^t \\&+d(\eta+\zeta)\prn*{\frac{1}{\gamma}+\frac{\ln(1/\gamma)}{\eta}+\frac{\log(d/\delta')}{2\zeta}}+\frac{\log(d)}{\eta}\\
    &\leq 
    \sum_t \ell_f^t+\frac{
6d\log(d^2/\delta)}{\eps} +d\prn*{1+2\ln(2d/\eps)+\log(d^2/\delta)}
\end{align*}
where the final inequality is derived by replacing the parameters $\gamma$, $\zeta$, $\eta$, and $\delta'$, and using the fact that $\frac{1-\eps}{1-\gamma d}+(\eta+\zeta)d\leq 1$ for the selection of the parameters.
\end{proof}

\subsection{Fixed feedback graphs using clique partition}\label{ssec:clique_partition}
We now extend the above guarantee to the graph-feedback setting providing a guarantee that scales with the minimum clique partition number $\kappa(G)$ when the graph is fixed. This generalizes the bandit setting where the graph is fixed and corresponds to the empty set. Compared to the black-box result of the previous section, the dependence on $L^{\star}$ here is optimal, however the analysis relies on the graph being fixed rather than evolving, and the bound scales with the minimum clique partition rather than the independence number.

To achieve the $\widetilde{\bigO}\prn*{\sqrt{\kappa(G) L^{\star}\log(d)}}$ regret bound, we use the black box framework with exponential weights as the full information engine for the estimated losses along with freezing.  The resulting algorithm  GREEN-IX-Graph combines features of the black box framework of Algorithm \ref{alg:black_box} with the bandit algorithm GREEN-IX, and adds an additional freezing level to keep estimated losses of all arms close. We use three
freezing thresholds:  $\gamma$ and $\gamma'$ like in Algorithm \ref{alg:black_box}, but define $\gamma =\Theta{(\epsilon/\kappa)}$ using the clique-partition number $\kappa$ in place of $\alpha$, and add an additional $\beta=\Theta{(\epsilon/d)}$ on the probabilities of arms being played (rather than observed). 
Like GREEN-IX, its graph version GREEN-IX-Graph uses the exponential weights algorithm as its 
full information learning algorithm with learning rate $\eta =\Theta(\epsilon/\kappa)$, and updates estimated losses via implicit exploration with a $\zeta=\Theta(\epsilon/\kappa)$ using the formula
$\widetilde{\ell}_i^t=\frac{\ell_i^t}{(W_i^t+\zeta)}$  if $i\in N_{I(t)}^t$ and 0 otherwise. The algorithm is formally defined in Algorithm \ref{alg:green_ix_graph}.

\begin{algorithm}[!h]
\caption{GREEN-IX-Graph}
\label{alg:green_ix_graph}
\begin{algorithmic}[1]
\REQUIRE Number of arms $d$, learning parameter $\epsilon'$, an estimate on $\kappa\ge\kappa(G)$, freezing thresholds $\beta:=\frac{\epsilon'}{d}$, $\gamma:=\frac{\epsilon'}{\kappa}$, $\gamma':=\frac{\gamma}{3}$, implicit exploration parameter $\zeta=\frac{\epsilon'}{6\kappa}$, full-feedback learning rate $\eta=\zeta$.
\STATE Initialize $p_i^1$ for arm $i$ uniformly ($p_i^t=1/d$) and set $t=1$ (round 1).
\FOR{$t=1$ \TO $T$}
\STATE Freeze arms whose probability is below $\beta$ to obtain:
$$
D^t=\crl*{i: p_i^t<\beta}
$$
\STATE Freeze arms whose observation probability is below $\gamma$, to obtain:
$$
F_0^t= \crl*{i
: \sum_{j \in N_i^t\backslash D^t
} p_j^t< \gamma}
$$

\STATE Recursively freeze remaining arms if their probability of being observed by non-frozen arms is below $\gamma'$ to obtain 
$F^t= D^t\cup\bigcup_{k\geq 0}F_{k}^t $ where, 
$$
F_k^t= \crl*{i\notin \prn*{\bigcup_{m=0}^{k-1} 
F_m^t}: \sum_{j\in \prn*{N_i^t\setminus\bigcup_{m=0}^{k-1} 
F_m^t}} p_j^t < 
\gamma'
}
$$
\STATE Normalize the probabilities of unfrozen arms so that they form a distribution.
$$
w_i^t=\begin{cases}
0 & \text{ if } i \in F^t \\
\frac{p_i^t}{1-\sum_{j\in 
F^t} p_j^t} & \text{ else }
\end{cases}
$$
\STATE Draw arm $I(t)\sim w^t$ and incur loss $\ell_{I(t)}^t$.
\STATE  Compute biased estimate of losses via implicit exploration with parameter $\zeta$:
$$
\widetilde{\ell}_i^t=\begin{cases}\frac{\ell_i^t}{W_i^t+\zeta} & \text{ if } i\in N_{I(t)}^t \backslash  F^t\\ 
0 & \text{ otherwise }
\end{cases}
$$
where $W^t_i=\sum_{j \in N^t_i} w_j^t$
\STATE Update $p_i^{t+1}$ via exponential weights update with learning rate $\eta$:~~ $p_i^{t+1}\propto p_i^t \exp(-\eta \widetilde{\ell_i^t})$. 
\ENDFOR
\end{algorithmic}
\end{algorithm}
We now show an optimal high-probability small loss guarantee using the clique-partition number:
\begin{theorem}\label{thm:clique_partition}
\emph{GREEN-IX-Graph} (Algorithm \ref{alg:green_ix_graph}) run with learning parameter $\eps'=\frac{\eps}{5}$ has the following regret bound with probability at least $(1-\delta)$:
\begin{align*}
  \regret(f)=\widetilde{\bigO}\prn*{ \sqrt{\kappa L^{\star} \log(d/\delta)}+\log(d/\delta)}
\end{align*}
\end{theorem}
\begin{proof}
The novel idea of the proof is to think of the exponential weight algorithm as running on two levels: selecting a clique $c$ in a clique partition on the top level, and then selecting an arm $i\in c$ in the clique. At the top level there are $\kappa$ options to choose from, and at the bottom level we are in a full information setting, any node $i\in c$ can observe all other nodes in $c$. 

To make this analysis work for clique partition, we need to overcome a few difficulties.
\begin{itemize}
\item We are running the exponential weight algorithm, importance sampling and freezing on the real graph $G$: the algorithm is not explicitly using the clique-partition $C$, instead we break the analysis into the two level structure suggested above. 
\item To show the high probability guarantee we show that the expected loss is closely tied with the loss using expected losses observed by the full information algorithm. In particular, we show that the algorithm's loss $\sum_t \ell_{I(t)}^t$ is very close to the following expression: $\sum_t \sum_i \frac{w_i^t}{W_i^t}(W_i^t+\zeta)\widetilde{\ell_i^t}$. It is not hard to see that the two expressions have the same expectation: $\sum_t \sum_i w_i \ell_i$. To show that they remain very close we use Lemma \ref{lem:black_box_high_probability} with $x_t=\ell_{I(t)}^t$ as well as  $x'_t=\sum_i\frac{w_i^t}{W_i^t}(W_i^t+\zeta)\widetilde{\ell_i^t}$. Note that $(W_i^t+\zeta)\widetilde{\ell_i^t}$ is either 0 or $\ell_i^t$, so at most 1, and hence we get that $x'_t \le \sum_i\frac{w_i^t}{W_i^t} \le \alpha(G)$, the independence number of $G$ as shown by \cite{AlonCGM13} for any probability distribution $w_i^t$.
\item We add the third threshold $\beta$ to make sure that the estimated losses of all arms remain close, a property used in low-loss guarantees for the case of bandits.
\end{itemize}

Based on this idea the proof follows a similar structure to the proofs of Theorems \ref{thm:black_box_whp} and \ref{thm:green_ix} but needs some extra care in i) introducing an extra freezing threshold based on probability of being played, ii) applying the concentration bounds, and iii) appropriately tackling the resulting second-order term. More formally, we first show the approximate regret guarantee compared to any arm $f$ with probability at least $1-\delta'$ where $\delta'=\frac{\delta}{d}$. To facilitate the exposition, we present it in steps.

\textbf{Concentration on actual losses.} We proceed by upper bounding the loss of the algorithm by a sum of weighted estimated losses. In that, we apply Lemma \ref{lem:black_box_high_probability} two times. First we relate the loss of the algorithm to its expected performance. With probability $(1-\frac{\delta'}{3})$ we have:
\begin{equation}
    \sum_t \ell^t_{I(t)} \le (1+\eps')\sum_t \sum_i w^t_i \ell^t_i + \frac{(1+\epsilon')\ln(3/\delta')}{\epsilon'}
\end{equation}
\textbf{Bounding total frozen probability.} Since there are at most $d$ arms, the total  probability mass that is frozen at line 3 is at most $\beta \cdot d= \frac{\eps'}{5}$. In the subsequent freezing steps, by Lemma \ref{lem:bounding_total_freezing} is at most $4\gamma \cdot \alpha\leq \frac{4\eps'}{5}$. Therefore the total probability frozen is $\eps'$ and for any non-frozen arm $i$: $w_i^t\leq \frac{p_i^t}{1-\eps'}$.
\begin{equation}
    (1+\eps')\sum_t \sum_i w^t_i \ell^t_i\leq \frac{1+\eps'}{1-\eps'}\sum_t \sum_i p^t_i \ell^t_i
\end{equation}
\textbf{Reduction to estimated losses.}
We move forward by connecting the performance of the full information algorithm to an analogous expression on estimated losses. The following expression holds deterministically by analyzing the estimated loss term:
\begin{equation}
    \frac{1+\eps'}{1-\eps'}\sum_t\sum_i p_i^t\ell_i^t = \frac{1+\eps'}{1-\eps'}\sum_t\En\brk*{\sum_i \frac{p_i^t}{W_i^t}(W_i^t+\zeta)\cdot \widetilde{\ell_i^t}}
\end{equation}
\textbf{Concentration on estimated losses.} We now wish to connect this expectation to its realization. Unfortunately estimated losses of different arms are conditionally dependent for the same round. Therefore we define the whole summation over the number of arms as our random variable. Note that $(W_i^t+\zeta)\cdot\widetilde{\ell_i^t}\leq 1$ by definition of the estimated loss and
$\sum_{i:W_i^t\neq 0} \frac{p_i^t}{W_i^t}\leq \sum_{i:W_i^t\neq 0} \frac{p_i^t}{P_i^t}\leq \alpha$ as observed in \cite{AlonCGM13}. Therefore we can apply the converse direction of Lemma \ref{lem:black_box_high_probability} with range of values at most $\alpha$, and obtain with probability $(1-\frac{\delta'}{3})$:
\begin{eqnarray}
    \nonumber\frac{1+\epsilon'}{1-\epsilon'}\sum_t \En\brk*{\sum_i \frac{p_i^t}{W_i^t}(W_i^t+\zeta)\widetilde{\ell_i^t}}\le \frac{(1+\epsilon')}{(1-\epsilon')^2}\prn*{\sum_t \sum_ip_i^t\widetilde{\ell_i^t}+\sum_t \sum_i \frac{p_i^t}{W_i^t}\zeta\widetilde{\ell_i^t }}+\\ \alpha\cdot \frac{(1+\epsilon')^2\ln(3/\delta')}{(1-\epsilon')^2\epsilon'}
\end{eqnarray}
\textbf{Second-order bound of exponential weights.} We now relate the performance of the full information algorithm on estimated losses to the cumulative estimated loss of a comparator $f$ via Lemma \ref{lem:mwu_second_order_regret}:
\begin{equation}
     \sum_t \sum_i p_i^t \widetilde{\ell_i^t }\leq \sum_t \widetilde{\ell_f^t}+\eta\sum_t \sum_i p_i^t (\widetilde{\ell_i^t})^2+\frac{\log(d)}{\eta}\le
     \sum_t \widetilde{\ell_f^t}+\eta \sum_t \sum_i \frac{p_i^t}{W_i^t} \widetilde{\ell_i^t}+\frac{\log(d)}{\eta}
\end{equation}
\textbf{Decomposing across cliques.} What is left is to relate the double summation of the weighted estimated losses to the estimated loss of the comparator. For that we first decompose across cliques. Let $C$ be such a clique partition of minimum size. For a clique $c\in C$, let $P_c^t$ be the total probability of the nodes in the clique. Note that, for any such node $i\in c$ which is non-frozen at round $t$, it holds that $W_i^t\geq P_i^t/3$. This is because we are in one of two scenaria: i) no node in the clique is frozen; in this case $W_i^t\geq P_i^t \geq P_c^t$, or ii) some node is frozen which implies that $P_c^t<\gamma$; however $i$ is not frozen and therefore $W_i^t\geq \gamma'=\frac{\gamma}{3}\geq P_c^t$. We therefore obtain:
\begin{equation}\label{eq:clique_decomposition}
    \sum_t \sum_i \frac{p_i^t}{W_i^t}\widetilde{\ell_i^t} = 
   \sum_t \sum_{c\in C} \sum_{i\in c} \frac{p_i^t}{W_i^t}\widetilde{\ell_i^t} 
    \le 3\cdot \sum_t \sum_{c\in C} \sum_{i\in c} \frac{p_i^t}{P_c^t}\widetilde{\ell_i^t} 
    = 3\cdot \sum_{c\in C} \sum_t \sum_{i\in c} \frac{p_i^t}{P_c^t}\widetilde{\ell_i^t}
\end{equation}
\textbf{Fictitious exponential weights within each clique.}
The key insight of the proof is that the quantity $\frac{p_i^t}{P_c^t}$ appearing in the previous inequality is the probability of playing arm $i$ if we commit to play one arm from the clique. The sum for a clique in the right most expression is the expected loss of the expoential weight algorithm running on the clique. This is because, with exponential weights:
$$
\frac{p_i^t}{P_c^t}=\frac{\exp\prn*{-\eta \sum_{s=1}^t\widetilde{\ell_i^s}}}{\sum_{j\in c}\exp\prn*{-\eta \sum_{s=1}^t\widetilde{\ell_j^s}}}.
$$
As a result we can again apply Lemma \ref{lem:mwu_second_order_regret}, utilizing that estimated losses are at most $1/\gamma'$. Let $f(c)$ denote  an arbitrary representative of the clique $c\in C$. Then:
\begin{equation*}
    \sum_t\sum_{i\in c} \frac{p_i^t}{P_c^t}\widetilde{\ell_i^t}\leq \sum_t \widetilde{\ell}_{f(c)}^t+\eta \sum_t \sum_{i\in c}\frac{p_i^t}{P_c^t}(\widetilde{\ell_i^t})^2+ \frac{\log(d)}{\eta}
    \leq \sum_t \widetilde{\ell}_{f(c)}^t+\frac{\eta}{\gamma'} \sum_t \sum_{i\in c}\frac{p_i^t}{P_c^t}\widetilde{\ell_i^t}+ \frac{\log(d)}{\eta}
\end{equation*}
This implies (since $\gamma'=2\eta$):
\begin{equation}
    \sum_t \sum_i \frac{p_i^t}{P_c^t}\widetilde{\ell_i^t}\leq \frac{\sum_t \widetilde{\ell}_{f(c)}^t+\frac{\log(d)}{\eta}}{1-\frac{\eta}{\gamma'}}\leq 2\sum_t \widetilde{\ell}_{f(c)}^t+2\frac{\log(d)}{\eta}.
\end{equation}
\textbf{Estimated losses close to each other.}
We now use the same idea of Lemma \ref{lem:estimated_losses_close} to show that the cumulative estimated losses of the representatives of the cliques are close to the cumulative estimated loss of the comparator $f$. To this end, for any $c\in C$, we define $\tau_c=\max\{t\leq T: p_{f(c)}^t>\beta\}$. 
Beginning from: $p_{f(c)}^{\tau_c}>\beta$ and applying the arguments in the proof of Lemma \ref{lem:estimated_losses_close} and the fact that estimated losses are still bounded by $1/\gamma'$ (since they are weighted by the probability of observation), we obtain:
\begin{equation}\label{eq:MW-close}
    \sum_t \widetilde{\ell}_{f(c)}^t\leq \sum_t \widetilde{\ell}_f^t+\frac{1}{\gamma'}+\frac{\ln(\frac{1}{\beta})}{\eta}
\end{equation}
\textbf{Concentration from estimated to actual loss via implicit exploration.} Last we need to relate the cumulative estimated loss of the comparator $f$ to its actual loss. This occurs directly by Lemma \ref{lem:implicit_exploration}, with probability at least $1-\frac{\delta'}{3}$
\begin{equation}
    \sum_{t=1}^T \widetilde{\ell_f^t}\leq \sum_{t=1}^T \ell_f^t+ \frac{\log(3/\delta')}{2\zeta} 
\end{equation}
\textbf{Putting everything together.}
Combining the numbered inequalities and setting $\eta=\zeta=\gamma'/2$ and $\epsilon'$ such that $\frac{(1+\eps')(1+6\zeta+6\eta}{(1-\eps')^2}\leq \frac{1}{1-\eps}$, and applying union bound on the failure probabilities concludes the proof for approximate regret. The proof for regret then follows doubling trick arguments similar as in previous theorems.
\end{proof}

\textbf{LP relaxation: Beyond Clique Partition.} In the proof above, we can replace the clique-partition number $\kappa(G)$ with the smaller linear programming relaxation for clique partition $\kappa_f(G)$. Unfortunately this fractional clique partition number is also computationally unfeasible to obtain. 

\begin{align*}
\kappa_f(G)= \min \sum_{c \textrm{ Clique in }G} y_c\\
\mathrm{s.t.}~ \forall i \in [d]~~&~~~~ \sum_{c \ni i} y_c \ge 1\\
 \forall c \textrm{ Clique in }G~~&~~~~ y_c \ge 0
\end{align*}
We call a solution  $y$ to the above linear program a fractional clique-partition, and  will use $y$'s to represent the solution.

We note that Algorithm \ref{alg:green_ix_graph} only needed a bound on $\kappa$ and did not use the clique-partition in the algorithm. We claim that we can also use the fractional clique-partition number $\kappa_f$ in place of $\kappa$ and the analogous theorem would hold:

\begin{theorem}\label{thm:f_clique_partition}
Algorithm \ref{alg:green_ix_graph} run with an appropriate parameter $\eps'$ using $\kappa_f$ in place of $\kappa$ above, has the following regret bound with probability at least $(1-\delta)$:
$$
  \regret(f)=\widetilde{\bigO}\prn*{ \sqrt{\kappa_f L^{\star} \log(d/\delta)}}
$$
\end{theorem}
\begin{proof}[Proof Sketch]
The proof is completely analogous to the proof of Theorem \ref{thm:clique_partition}: the only change occurs in Eq. \eqref{eq:clique_decomposition} where one needs to
replace summing over cliques $\sum_{c\in C}$ in each sum with a weighted sum of the cliques used in the clique partition $\sum_{c \in C} y_c$. The first equation now becomes an inequality
$$
    \sum_t \sum_i \frac{p_i^t}{W_i^t}\widetilde{\ell_i^t} \le  
   \sum_t \sum_{c\in C} y_c\sum_{i\in c} \frac{p_i^t}{W_i^t}\widetilde{\ell_i^t} 
   $$
due to the constraint that $\sum_{c \ni i} y_c \ge 1$. The rest of the proof follows as before. 
\end{proof}

\section{Semi-bandits, contextual bandits, and shifting comparators} \label{sec:applications}
In this section, we extend our framework to three important applications defined in Section \ref{ssec:appl_models}. In Section~\ref{ssec:semi_bandits_black_box}, we build on Section~\ref{sec:black_box} to provide a black-box framework for combinatorial semi-bandits (modeling settings such as online routing) and subsequently also show the first high-probability optimal small-loss bound for this setting. In Section~\ref{ssec:contextual_bandits_black_box}, we discuss contextual bandits (including applications with infinite comparator classes). Finally, in Section~\ref{ssec:shifting_comparators}, we study learning against shifting comparators and discuss the consequences of our results to game theory.

\subsection{Combinatorial semi-bandits}\label{ssec:semi_bandits_black_box}
\newcommand{\mathsc}[1]{{\normalfont\textsc{#1}}}

To model combinatorial semi-bandits as a variant of our feedback graph framework, we construct a bipartite graph with nodes $\mathcal{F}$ and $\mathcal{E}$, and connect strategies $f$ to the elements included in $f$. We note that this graph does not need to be explicitly maintained by the algorithm as we discuss below and elaborate upon in Appendix \ref{app:sampling}.

Similarly to Section~\ref{sec:black_box}, we provide a reduction from full information to partial information for this setting. The full-information algorithm takes as input a loss vector (corresponding to losses of all elements in $\mathcal{E}$) and outputs a probability distribution over all actions in $\mathcal{F}$. In particular, it runs on estimated losses created by importance sampling as before and induces, at round $t$, a probability distribution $p^t$ on the set of strategies $\mathcal{F}$ as only strategies can be selected and not individual elements. We assume a bound on the expected approximate regret of $B\prn*{\eps,T,\mathcal{F}}$ for the full information algorithm when losses are in $[0,1]$. This will scale linearly with the magnitude of losses $L$ of an action $f\in\mathcal{F}$.
We apply importance sampling and freezing to the elements $e\in \mathcal{E}$. The probability of observing an element is the sum of the probabilities of the adjacent strategies. This shows the leverage freezing offers in bounding the number of samples required. We modify the freezing process to 
freeze elements when observed with probability less than $\gamma$, and then freeze all strategies that contain a frozen element. We note that computationally efficient algorithms~\footnote{These algorithms can be implemented efficiently with a linear optimization oracle by adding independent perturbation on the element set $\mathcal{E}$. In contrast, algorithms like Exponential Weights need to keep a weight for each action in $\mathcal{F}$.} such as \emph{Follow the Perturbed Leader} \cite{Hannan,KalaiVempala} do not output directly the probabilities of each strategy but only an action drawn from the corresponding distribution. Nevertheless, we can still use these algorithms as we can estimate these probabilities (see discussion after Theorem~\ref{thm:semisub}).

The reduction (Algorithm~\ref{alg:black_box_semibandits}) is similar to the one of Algorithm \ref{alg:black_box}. In the corresponding initial step and recursive process, we only freeze nodes that are in the set $\mathcal{E}$ if their observation probability is below the threshold~\footnote{This observation probability can be computed by resampling strategies multiple times as briefly discussed below and elaborated upon in Appendix~\ref{app:sampling}.} and apply a single threshold $\gamma=\frac{\eps'}{\abs{\mathcal{E}}}$ for all the recursive steps (instead of multi-thresholding). We subsequently freeze any node in $\mathcal{F}$ that is adjacent to a frozen node in $\mathcal{E}$, and repeat the recursive process until no unfrozen element $e$ has probability of observation smaller than $\gamma$. We refer to the set of frozen elementsand strategies as $F^t_{\mathsc{Elem}}$ and $F^t_{\mathsc{Str}}$ respectively. After the freezing process, the final probability distribution $w^t$ is derived again via a renormalization on the non-frozen strategies (as in step \ref{step_normalize_black_box} of Algorithm \ref{alg:black_box}).
\begin{algorithm}[!h]
\caption{Black-box combinatorial semi-bandit freezing algorithm}
\label{alg:black_box_semibandits}
\begin{algorithmic}[1]
\REQUIRE Full information algorithm $\mathcal{A}$, element and strategy sets $\mathcal{E}$ and $\mathcal{F}$, learning parameter $\epsilon'$, freezing threshold $\gamma:=\frac{\epsilon'}{\abs{\mathcal{E}}}$.
\STATE Initialize $p_i^1$ for arm $i$ based on the initialization of $\mathcal{A}$ and set $t=1$ (round 1). 
\FOR{$t=1$ \TO $T$}
\STATE  Initialize the freezing sets $F_{\mathsc{Elem}}^t=\emptyset$ and $F_{\mathsc{Str}}^t=\emptyset$ .
\STATE Recursively freeze elements
observed with probability below $\gamma$ and actions containing them:
\begin{align*}
F_{\mathsc{Elem}}^t &\gets F_{\mathsc{Elem}}^t\cup \crl*{e\in\mathcal{E}: P_e^t(F_{\mathsc{Str}}^t)<\gamma}\\
F_{\mathsc{Str}}^t &\gets F_{\mathsc{Str}}^t\cup \crl*{i\in\mathcal{F}: \exists e\in\mathcal{E}\cap F_{\mathsc{Str}}^t 
\text{ s.t. }e\in i}.
\end{align*}
estimating $P_e^t(F_{\mathsc{Str}}^t)=\sum_{i\in\mathcal{F}\backslash F_{\mathsc{Str}}^t, e\in i}p_i^t$ using sampling described in Appendix~\ref{app:sampling}.
\STATE Normalize the probabilities of unfrozen actions $i\in \mathcal{F}$ so that they form a distribution.
$$
w_i^t=\begin{cases}
0 & \text{ if } i \in F_{\mathsc{Str}}^t \\
\frac{p_i^t}{1-\sum_{j\in 
F^t} p_j^t} & \text{ else }
\end{cases}
$$
\STATE Draw arm $I(t)\sim w^t$ and incur loss $\ell_{I(t)}^t$.
\STATE Compute estimated loss for all elements $e\in\mathcal{E}$:
$$
\widetilde{\ell}_e^t=\begin{cases}\frac{\ell_e^t}{W_e^t} & \text{ if } e\in I(t)\\
0 & \text{ else }
\end{cases}.
$$
\vspace{-0.2em}
where $W^t_e=\sum_{i: e\in i} w_i^t$
\STATE Update $p_i^{t+1}$ using full information algorithm $\mathcal{A}$ with
loss $\widetilde{\ell}^t$ for round $t$.\label{step:update_semi_black_box}
\ENDFOR
\end{algorithmic}
\end{algorithm}

We now provide the equivalent lemma to Lemma \ref{lem:bounding_total_freezing} to bound the total frozen probability.
\begin{lemma}\label{lem:freezing_semi_bandits}
\vspace{0.1in}
At round $t$, the total probability that is $F^t_{\mathsc{Elem}}\subset \mathcal{F}$ is at most $\eps'$: $\sum_{i\in F_{\mathsc{Str}}^t} p_i^t\leq \eps'$.
\end{lemma}
\begin{proof}
When a node in $\mathcal{E}$ becomes frozen in the initial step, it means that its probability of observation is less than $\gamma$.  Since the probability of playing adjacent nodes in $\mathcal{F}$ contributes to this probability of observation, at the initial step of the recursive process, the total probability frozen is less than $\gamma$ times the number of nodes in $\mathcal{E}$ that are frozen. As before, freezing some nodes in $\mathcal{E}$, may cause other nodes to become frozen. By freezing $e\in \mathcal{E}$, we also freeze all its neighbors  $\mathcal{F}$, which can decrease the observation probability of other elements.
In the propagation process, if an element 
becomes frozen its total probability of observation by not already frozen strategies is at most $\gamma$. Hence the total frozen probability is at most $\gamma\cdot \abs{\mathcal{E}}$ which concludes the lemma.
\end{proof}

\textbf{Black-box result.} We now provide a small-loss regret guarantee assuming that the full-information algorithm used in Step \ref{step:update_semi_black_box} of Algorithm~\ref{alg:black_box_semibandits} has an approximate regret bound of $B(\eps,T,\mathcal{F})$. We note that computationally efficient versions of \emph{Follow the Perturbed Leader} enjoy a bound $B(\eps,T,\mathcal{F})=\bigO\prn*{m \log(\abs{\mathcal{E}})/ \eps}$ (e.g., Theorem 1.1b in \cite{KalaiVempala}).
We note that we apply freezing on elements based on their probability of observation (and freeze elements with observation probability below $\gamma$). Moreover, the estimated losses are computed via importance sampling and each strategy includes at most $m$ elements. As a result, the magnitude of the estimated losses of a strategy is at most $L=m/\gamma$ where $m$ corresponds to the maximum number of elements in a strategy, e.g.,  the maximum length of any path. 
\begin{theorem}\label{thm:semisub}
\vspace{0.1in}
Let $\mathcal{A}$ be any full information algorithm for the problem whose expected approximate regret is bounded as $\apx(f,\eps/3) \le L\cdot B\prn*{\eps,T,\mathcal{F}}$ when run on losses bounded by $L$. 
Then, the \emph{Semi-Bandit Freezing Algorithm} (Algorithm~\ref{alg:black_box_semibandits}) run with learning rate $\eps'=\eps/3$ on input $\mathcal{A}$ guarantees that for any $\delta>0$ with probability $1-\delta$,  
$$
\forall f \in \mathcal{F},~~ \apx(f,\eps)=\frac{9m|\mathcal{E}|\log\prn*{ \abs{\mathcal{E}}/\delta}}{\eps^2}+\frac{3m|\mathcal{E}| B\prn*{\eps,T,\mathcal{F}}}{\epsilon}
$$ 
\end{theorem}
\begin{proof}
The proof follows similarly as the one of Theorem \ref{thm:black_box_whp} adjusted to the semi-bandit setting.
We denote by $W_e^t$ the probability of observing an element $e$. Also we use the subscript $i$ for strategy nodes (paths) and the subscript $e$ for element nodes (edges). Recall that $m$ is the maximum number of edges in any path. More formally, for each comparator $f\in\mathcal{F}$, we obtain the following set of inequalities with probability at least $1-\delta'$: 
\begin{align*}
\sum_t \ell_{I(t)}^t &=\sum_t \sum_{e\in I(t)} \ell_e^t=\sum_t \sum_{e\in I(t)}W_e^t \widetilde{\ell_e^t}=\sum_t \sum_{e \in \mathcal{E}} W_e^t \widetilde{\ell_e^t}=\sum_t \sum_{i \in \mathcal{F}} w_i^t \widetilde{\ell_i^t}\\
&\leq \frac{1}{1-\eps'}\sum_t \sum_{i \in \mathcal{F}} p_i^t  \sum_{e\in i} \widetilde{\ell_e^t} & \textrm{Using Lemma \ref{lem:freezing_semi_bandits}}
\intertext{Using the full information guarantee and noting that losses are at most $L=\frac{m}{\gamma}$, this is bounded by}
& \leq \frac{1}{(1-\eps')^2}\prn*{\sum_t \widetilde{\ell_f^t}+m\prn*{\frac{B\prn*{\eps',T,\mathcal{F}}
}{\gamma}}}\\
\intertext{Now by applying concentration Lemma \ref{lem:black_box_high_probability}, for each $f$ and taking a union bound over $f \in \mathcal{F}$,  }
& \leq \frac{\prn*{1+\eps'}}{\prn*{1-\eps'}^2}\prn*{\sum_t \ell_f^t+\prn*{\frac{\ln\prn*{|\mathcal{F}|/\delta'}}{\gamma \cdot\eps'}+\frac{
m B\prn*{\eps',T,\mathcal{F}}}{\gamma}}}
\end{align*}
Since $|\mathcal{F}| \le |\mathcal{E}|^m$, using $\gamma = \frac{\epsilon'}{|\mathcal{E}|}$ and $\epsilon'=\frac{\eps}{3}$ such that $\frac{1+ \epsilon'}{(1 - \epsilon')^2}\leq (1-\eps)^{-1}$, we conclude the proof.
\end{proof}

\textbf{Sampling the probabilities of observation.}
In the previous part, we assumed that, at any point, we have access to the probability $P_e^t(F_{\mathsc{Str}}^t)$ that an element is observed. This is used both to define  which elements are frozen and to define the estimated loss $\widetilde{\ell_e^t}$. Note that algorithms for semi-bandits such as Follow the Perturbed Leader do not provide directly these probabilities, 
but instead maintain weights on elements only, and offer a method to sample the strategies using these weights.
This assumption can be removed by appropriate sampling. More formally, we first create estimates on the observation probabilities of all the elements via drawing actions from the full information algorithm.  If the elements have observation probability less than $\gamma$  then we freeze them as in the recursive steps of the algorithm. This process could in principle require many samples  to obtain such estimates. However, freezing provides leverage since we do not need to compute exact estimates but we are fine if we have established whether they are i) with high probability greater than $\gamma/2$ if we do not freeze them (so that we use $2/\gamma$ as a bound on the magnitude and ii) with high probability less than $2\gamma$ as we then just need to set $\eps$ half of what we discussed before to still get the same bound on the total probability of being frozen. This task requires $\widetilde{\bigO}\prn*{\frac{1}{\eps^2\gamma}}$ steps with high probability. Since there is at most a $1-\eps$ probability that is not frozen, the samples that need to be discarded as they include frozen arms are rare and do not affect the high probability guarantee. We provide details of this argument in Appendix \ref{app:sampling}.

\textbf{Optimal $\sqrt{L^{\star}}$ high-probability guarantee.}
In order to obtain an improved guarantee for semi-bandits, we need algorithm-specific arguments to address the inefficiencies in the black-box analysis. Neu \cite{Neu15_semibandits} provided an adaptation of the \emph{FPL} algorithm with optimal small-loss pseudo-regret guarantee in expectation which he termed \emph{FPL-TrIX}. This adaptation makes two modifications: implicit exploration and truncation. More formally, for some learning rate $\eta$, the algorithm selects the strategy with the minimum  cumulative perturbed loss, where the cumulative losses of all elements are perturbed based on a truncated exponential distribution. In particular, for $B=\log(d/m)-\log(\eta)$, each element's perturbation is drawn from a distribution with pdf:
\begin{equation}\label{eq:truncated_perturbation}
    f(z)= \left\{\begin{array}{ll} 
    \frac{e^{-z}}{1-e^{-B}} & \text{if } z\in [0,B]\\
    0 &  \text{otherwise}   
    \end{array} \right.
\end{equation}
In Algorithm~\ref{alg:sampler_fpltrix} we provide a procedure that takes as input a set of (unfrozen) elements $U$ and samples a strategy based on  Follow the Perturbed Leader procedure consisting of only unfrozen elements belonging to the set $U$. We note that this subroutine remains computationally efficient precisely because the feasible set is defined by removing individual elements $e \in \mathcal{E}$, so the original optimization oracle can be still used on a reduced element set. This would not be the case if the feasible set would be defined through removing strategies $i\in\mathcal{F}$ arbitrarily (e.g., for a shortest path problem, removing only certain paths instead of eliminating all paths including certain edges).

This procedure  is a key subroutine for our main algorithm \emph{FPL-TrIX with freezing} (Algorithm~\ref{alg:fpltrix_freezing}). The main idea is that we use the sampler to estimate probabilities of elements being selected based on which we perform recursive freezing. Next, using only the set of unfrozen elements in the sampler, we sample the strategy our main algorithm finally produces for that round. 

\begin{algorithm}[!h]
\caption{Sampler using Follow the Perturbed leader with truncations (samples a strategy)}
\label{alg:sampler_fpltrix}
\begin{algorithmic}[1]
\REQUIRE Learning rate $\eta$, set of unfrozen elements $U$, cumulative estimated losses $\widetilde{L}_e$ $\forall e\in U$.
\STATE {\bf Sample($U, \{\widetilde{L}_e : e \in U\}, \eta$)}
\STATE ~~~~~ Sample perturbations $Z_e$ according to a distribution with pdf defined as in Eq. \eqref{eq:truncated_perturbation}.
\STATE ~~~~~  Set $\mathcal{F}' = \{f \in \mathcal{F}: f \cap U = f \}$ 
\RETURN $\argmin_{f \in \mathcal{F}'} \sum_{e \in f}\left\{\eta\widetilde{L}_e - Z_e\right\}$
\end{algorithmic}
\end{algorithm}

\begin{algorithm}[!h]
\caption{FPL-TrIX with freezing}
\label{alg:fpltrix_freezing}
\begin{algorithmic}[1]
\REQUIRE Element and strategy sets $\mathcal{E}$ and $\mathcal{F}$, learning parameter $\epsilon'$, freezing threshold $\gamma:=\frac{\epsilon'}{\abs{\mathcal{E}}}$, implicit exploration parameter $\zeta:=\frac{\eps'}{2\abs{\mathcal{E}}}$, $m:=\max_{f\in \mathcal{F}}\abs{\{e:e\in \mathcal{F}\}}$, FPL parameter $\eta:=\frac{\epsilon'}{4m\abs{\mathcal{E}}}$.
\STATE Initialize $\widetilde{L}_e^0=0$ for all elements $e\in\mathcal{E}$ and set $t=1$ (round 1). 
\FOR{$t=1$ \TO $T$}
\STATE  Initialize the freezing sets $F_{\mathsc{Elem}}^t=\emptyset$ and $F_{\mathsc{Str}}^t=\emptyset$ .
\STATE Recursively freeze elements
observed with probability below $\gamma$ and actions containing them:
\begin{align*}
F_{\mathsc{Elem}}^t &\gets F_{\mathsc{Elem}}^t\cup \crl*{e\in\mathcal{E}: P_e^t(F_{\mathsc{Str}}^t)<\gamma}\\
F_{\mathsc{Str}}^t &\gets F_{\mathsc{Str}}^t\cup \crl*{i\in\mathcal{F}: \exists e\in\mathcal{E}\cap F_{\mathsc{Str}}^t 
\text{ s.t. }e\in i}.
\end{align*}
estimating the observation probability $P_e^t(F_{\mathsc{Str}}^t)$ of an element $e\in\mathcal{E}\setminus F_{\mathsc{Elem}}^t$, using the Sample (Algorithm~\ref{alg:sampler_fpltrix}) on set $U^t=\mathcal{E}\setminus F_{\mathsc{Elem}}^t$, i.e., $Sample(U^t, \crl{\widetilde{L}_e^{t-1}}, \eta)$.
\STATE Estimate the probabilities of observation for each element after freezing:
$$
W_e^t=\begin{cases}
0 & \text{ if } e \in F_{\mathsc{Elem}}^t \\
P_e^t(F_{\mathsc{Str}}^t) & \text{ otherwise }
\end{cases}
$$
\STATE Draw arm $I(t)=Sample(\mathcal{E}\setminus F^t,\crl{\widetilde{L}_e^{t-1},\eta})$ and incur loss $\ell_{I(t)}^t$.
\STATE Compute biased estimate of losses for all elements $e\in\mathcal{E}$ via implicit exploration $\zeta=\gamma$:
$$
\widetilde{\ell}_e^t=\begin{cases}\frac{\ell_e^t}{W_e^t+\zeta} & \text{ if } e\in I(t)\\
0 & \text{ else }
\end{cases}.
$$
\STATE Update the cumulative estimated loss: $\widetilde{L}_e^t=\widetilde{L}_e^{t-1}+\widetilde{\ell}_e^t$.
\ENDFOR
\end{algorithmic}
\end{algorithm}

\begin{lemma}[implied by Lemmas 6, 7, and 8, followed by Lemma 1 in \cite{Neu15_semibandits}]\label{lem:trufpl_full_info}
\vspace{0.1in}
FPL-TrIX run on estimated losses satisfies the following guarantee for any $f\in \mathcal{F}
$:
$$
\sum_t\sum_{g\in \mathcal{F}}p_g^t \widetilde{\ell}_g^t\leq \sum_t \widetilde{\ell}_f^t +2 \eta \cdot m \sum_t \sum_{j\in\mathcal{E}} \widetilde{\ell}_j^t+ \frac{m\log\prn*{|\mathcal{E}|/m}+1}{\eta}
$$
\end{lemma}
\begin{proof}
Following the notation of \cite{Neu15_semibandits}, we define $\widetilde{Z}$ to be a fixed exponentially distributed perturbation and $Z^{t}$ to be the truncated exponentially distributed perturbation drawn at round $t$ in step 6 of Algorithm~\ref{alg:fpltrix_freezing}. The selected action of our algorithm at round $t$ is $I(t)=\argmin_{g\in \mathcal{F}} \sum_{e\in g}\prn*{\eta\widetilde{L}^{t-1}_e-Z_e^t}$ and let also $\widetilde{I}(t)=\argmin_{g\in \mathcal{F}}  \sum_{e\in g}\prn*{\eta\widetilde{L}^{t-1}_e-\widetilde{Z}_e}$. Letting $\mathbb{E}_{t-1}$ denote the conditional expectation conditioned on all the information before round $t$, we let $p_g^t=\mathbb{E}_{t-1}\brk*{\mathbf{1}\crl*{I(t)=g}}$ and $\widetilde{p}_g^t=\mathbb{E}_{t-1}\brk*{\mathbf{1}\crl*{\widetilde{I}(t)=g}}$ for all $g\in\mathcal{F}$, while $q_e^t=\mathbb{E}_{t-1}\brk*{\mathbf{1}\crl*{e\in I(t)}}$ and $\widetilde{q}_e^t=\mathbb{E}_{t-1}\brk*{\mathbf{1}\crl*{e\in \widetilde{I}(t)}}$ for all $e\in\mathcal{E}$. Lemmas 6, 7, and 8 in \cite{Neu15_semibandits} directly provide the following bound:
\begin{align}\label{eq:semibandits_fixed_truncation}
\sum_t\sum_{g\in \mathcal{F}}\widetilde{p}_g^t \widetilde{\ell}_g^t\leq \sum_t \widetilde{\ell}_f^t + \eta \cdot m \sum_t \sum_{j\in\mathcal{E}} \widetilde{\ell}_j^t+ \frac{m\log\prn*{d/m}+1}{\eta}
\end{align}
Lemma 1 of \cite{Neu15_semibandits} proves that, for all $e\in\mathcal{E}$, $|q_e^t-\widetilde{q}_e^t|\leq \beta d$ where $d=|\mathcal{E}|$ and $\beta=e^{-B}=\frac{\eta\cdot m}{d}$. We can therefore rewrite the LHS of \eqref{eq:semibandits_fixed_truncation} as:
\begin{align}\label{eq:semibandits_changing_truncations}
\sum_t\sum_{g\in \mathcal{F}}\widetilde{p}_g^t \widetilde{\ell}_g^t= \sum_t\sum_{e\in\mathcal{E}}\widetilde{q}_e^t\widetilde{\ell}_e^t \geq  \sum_t\sum_{e\in\mathcal{E}}q_e^t\widetilde{\ell}_e^t-\beta d\sum_t\sum_e \widetilde{\ell}_e^t.
\end{align}
Combining Eq.~\eqref{eq:semibandits_fixed_truncation} and \eqref{eq:semibandits_changing_truncations} with the definition of $\beta=\frac{\eta\cdot m}{d}$, the lemma follows.
\end{proof}
To ensure that the cumulative estimated losses are not too far from each other, Neu truncated the perturbations that are higher than some parameter. The effect of this truncation is similar to one of the effects of freezing: if two strategies differ significantly in their cumulative loss, adding truncated noise does not change their order, so the higher loss strategy is not selected.
By using his algorithm he shows Lemma \ref{lem:estimated_losses_close_semi_bandits} which can be viewed as an equivalent of Lemma \ref{lem:estimated_losses_close}. This addresses the first inefficiency.

\begin{lemma}[Lemma 2 in \cite{Neu15_semibandits}]\label{lem:estimated_losses_close_semi_bandits}
\vspace{0.1in}
\emph{FPL-TrIX} run on the estimated losses 
guarantees that, for any element $j\in \mathcal{E}$ and strategy $g\in \mathcal{F}$, 
$$
\sum_{t=1}^T \widetilde{\ell}_j^t\leq \sum_{t=1}^T \sum_{e\in g}\widetilde{\ell}_{e}^t+\frac{m\log\prn*{|\mathcal{E}|/m}}{\eta}+\widetilde{\ell}_j^{T_j}
$$
where $T_j\le T$ is the last time that element $j$ had non-zero probability.
\end{lemma}
We now explain how freezing can enable transforming the above guarantees to high-probability bounds. Neu uses the technique of geometric resampling \cite{NeuB13} to create estimators that are equal in expectation to the ones developed by importance sampling (without actually computing the probabilities of each element). This technique works well in expectation but does not concentrate which creates a roadblock in providing high-probability guarantees. Instead, one can use actual sampling to create estimates close to these probabilities. With no lower bound on the targeted probability, the number of samples required for this purpose can be, in principle, unbounded. Freezing addresses this point by giving a lower bound on any probability of interest, and hence guaranteeing an upper bound on the number of samples required as established in Appendix \ref{app:sampling}. Combining this with the implicit exploration technique (as in \cite{Neu2015_implicit}) which, as before, addresses the inefficiency in the negative bias of the estimated loss of the comparator (Lemma \ref{lem:implicit_exploration}), we provide the following optimal high-probability approximate regret guarantee (Theorem \ref{thm:optimal_semi_bandits}). 

\begin{theorem}\label{thm:optimal_semi_bandits}
\vspace{0.1in} 
For $\epsilon<\frac{1}{2}$, \emph{FPL-TrIX with freezing} (Algorithm~\ref{alg:fpltrix_freezing}) run with learning parameter $\eps'=\eps/2$ guarantees for any $\delta>0$, with probability at least $1-\delta$, $\eps$-approximate regret at most
$$\frac{4m|\mathcal{E}|\log(|\mathcal{E}|/\delta)+24m^2|\mathcal{E}|\log(|\mathcal{E}|/m)+(8m+2)\abs{\mathcal{E}}}{\eps}
$$.
\end{theorem}
\begin{proof}
Let $W_e^t$ and $P_e^t$ be the probability that $e\in\mathcal{E}$ is observed after and before freezing respectively, i.e. the sum of the probabilities that a strategy $g\in\mathcal{F}$ containing $e$ is used. 
\begin{align*}
    \sum_t \ell_{I(t)}^t &= \sum_t\sum_{e\in \mathcal{E}}(W_e^t+\zeta)\widetilde{\ell_e^t} ~~~~~~~~~~~~~~~~~~~~~~~~~~~~~~~~~~~~~~~~~~~~~~~~~~~~\textrm{(by importance weighting)} \\
    &\leq \frac{1}{1-\eps'
    }\sum_t \sum_{e\in\mathcal{E}} P_e^t \widetilde{\ell_e^t}+\zeta \cdot \sum_t \sum_{e\in\mathcal{E}}\widetilde{\ell_e^t} ~~~~~~~~~~~~~~~~~~~~~~~~~~~~~~~~~~~~~~~~~~~~~~~~\textrm{(by Lemma \ref{lem:freezing_semi_bandits})}\\&=\frac{1}{1-\eps'
    }\sum_t\sum_{g\in \mathcal{F}} p_g^t \widetilde{\ell_g^t}+\zeta \cdot \sum_t \sum_{e\in\mathcal{E}}\widetilde{\ell_e^t}\\
    &\leq \frac{1}{1-\eps'
    }\prn*{\sum_t \widetilde{\ell}_f^t +(2\eta m+\zeta)}\cdot \sum_t\sum_{j\in\mathcal{E}} \widetilde{\ell_j^t}+\frac{m\log(|\mathcal{E}|/m)+1}{\eta}
    ~~~~~~~~~~~~~~ (\textrm{by Lemma~\ref{lem:trufpl_full_info}})\\ 
    \intertext{Interchanging the sums:}
    &\leq \frac{1}{1-\eps'
    }\prn*{\sum_t \widetilde{\ell}_f^t +(2\eta m+\zeta)}\cdot \sum_{j\in\mathcal{E}}\sum_t \widetilde{\ell_j^t}+\frac{m\log(|\mathcal{E}|/m)+1}{\eta}\\
    \intertext{Applying Lemma \ref{lem:estimated_losses_close_semi_bandits}:}
     &\leq \frac{1}{1-\eps'
    }\prn*{\sum_t \widetilde{\ell}_f^t +(2\eta m+\zeta)}\cdot \sum_{j\in\mathcal{E}}\prn*{\sum_t \sum_{e\in f}\widetilde{\ell}_e^t+\frac{m\log(|\mathcal{E}|/m)}{\eta}+\frac{1}{\gamma}}
    +\frac{m\log(|\mathcal{E}|/m)+1}{\eta}\\
    &\leq \frac{1}{1-\eps'
    }\prn*{\sum_t \widetilde{\ell}_f^t +(2\eta m+\zeta)}|\mathcal{E}|\cdot\prn*{\sum_t \widetilde{\ell}_f^t+\frac{m\log(|\mathcal{E}|/m)}{\eta}+\frac{1}{\gamma}}+\frac{m\log(|\mathcal{E}|/{m})+1}{\eta}\\
    \intertext{Applying Lemma~\ref{lem:implicit_exploration}:}
    &\leq \frac{\prn*{(2\eta m+\zeta)|\mathcal{E}|+1}}{1-\eps'
    }\prn*{\sum_t \ell_f^t +\frac{m\log(|\mathcal{E}|/\delta)}{2\zeta}+\frac{m\log(|\mathcal{E}|/m)}{\eta}+\frac{1}{\gamma}}+\frac{m\log(|\mathcal{E}|/{m})+1}{\eta}
\end{align*}
The theorem then follows by substituting the
parameters of the algorithm and statement.
\end{proof}

Since we have such a high-probability guarantee, we can apply doubling trick similarly as in Theorem \ref{thm:black_box_small_loss}
and derive the optimal 
small-loss high-probability guarantee answering the open question of Neu \cite{Neu15_semibandits,Neu2015_implicit}. We can therefore obtain the following small-loss guarantee.

\begin{theorem}\label{cor:semi_bandits_opt}
\vspace{0.1in}
The above algorithm applied with doubling trick on parameter $\eps$ has regret of
$\bigO\prn*{\sqrt{(L^{\star}+1)\prn*{m|\mathcal{E}|\log(\frac{\abs{\mathcal{E}}}{\delta})+m^2|\mathcal{E}|\log(\frac{\abs{\mathcal{E}}}{m})+1}\cdot \log(L^{\star}+1)}+m|\mathcal{E}|\log(\frac{\abs{\mathcal{E}}}{\delta})+m^2|\mathcal{E}|\log(\frac{\abs{\mathcal{E}}}{m})}$ with probability at least $1-\delta$.
\end{theorem}

\subsection{Contextual bandits}
\label{ssec:contextual_bandits_black_box}
Contextual bandits can be seen as a direct application of the graph based feedback learning problem. The nodes of our graph are the policies, and two policies $f$ and $f'$ are connected by an edge at time $t$ if they recommend the same action in the context of time $t$, that is if $f(x_t)=f'(x_t)$. 
The feedback graph $G^t$ in this case is changing at each time step, but it is always a set of disjoint cliques. If the feedback graph consists of just cliques, we do not need the recursive freezing step of Algorithms \ref{alg:black_box} or \ref{alg:black_box_semibandits}, as a node in $G^t$ becomes frozen together with the whole clique it is contained in: effectively, we are freezing an action in each step if the probability mass of the policies recommending the action is below $\gamma$.

As a full information algorithm for this problem, we can use the oracle-efficient contextual bandit algorithms of \cite{Rakhlin2016_bistro,Syrgkanis2016_perturbed,vasilis_nips}, where oracle-efficient refers to the fact that the algorithm chooses an action using an oracle without needing to keep track of information for each policy, which allows it deal with large policy sets. Such algorithms has approximate regret at most $B\prn*{\eps,T,\mathcal{F}}=\sqrt{T\log\prn*{\abs{\mathcal{F}}}}$
when losses are in $[0,1]$.
The magnitude of the estimated losses is $\frac{1}{\gamma}$. We note that the above papers make some assumptions on the way the contexts are coming (the contexts are either known in advance or coming from a known distribution). When such a full information algorithm exists, we obtain the following guarantee.

\begin{theorem}\label{thm:contextual}
\vspace{0.1in}
Assume an oracle-efficient full-information algorithm $\mathcal{A}$ with $\eps$-approximate regret $B(\eps,T,\mathcal{F})=\sqrt{T\log(\abs{F})}$ when applied on losses in $[0,1]$. Then our freezing algorithm run with learning parameter $\eps'=\frac{\eps}{2}$ on input $\mathcal{A}$  guarantees that, for any $\delta>0$, with probability $1-\delta$, an $\eps$-approximate bound of
$\apx(f,\eps)=\bigO\prn*{\frac{d\cdot \sqrt{T\log\prn*{\abs{\mathcal{F}}}}}{\epsilon} +\frac{d\cdot \log\prn*{\abs{\mathcal{F}}/\delta}}{\epsilon^2}}$ for all $f \in \mathcal{F}$.
\end{theorem}

The proof follows by exactly the same ideas as before and is provided in Appendix \ref{ssec:proofs_contextual}. Applying the doubling trick as above, we can create a guarantee that is partly data-dependent, i.e. $\widetilde{\bigO}\prn*{\sqrt{T}+T^{1/4}\prn*{L^{\star}}^{1/2}}$. If we use as a full information algorithm the algorithm of Syrgkanis et al. \cite{vasilis_nips}, this guarantee becomes $\widetilde{\bigO}\prn*{\sqrt{T}+T^{1/3}\prn*{L^{\star}}^{1/3}}$ which improves on the best known bound of $\widetilde{\bigO}{\prn*{T^{2/3}}}$ established by their paper. 

We note that, by using multiplicative weights  
as a full-information algorithm, one can derive a $\widetilde{\bigO}\prn*{\prn*{L^{\star}}^{2/3}}$ guarantee that is however inefficient as it needs to keep weights for every policy. This was sketched in the open problem of Agarwal et al. \cite{pmlr-v65-agarwal17a} who asked for an oracle-efficient version of the above bound. Although we do not provide such a bound, our result improved on the best known bounds by replacing some of the dependence on the time horizon by $L^{\star}$. Subsequently to our work, this result was enhanced via a more sophisticated freezing of suboptimal arms, which uses multiple experts with different freezing thresholds \cite{AllenZBL2018}.

\textbf{Infinite comparator class $\mathcal{F}$.}
The black-box analysis need not be restricted to finite set $\mathcal{F}$ of policies. One can also consider a, potentially uncountable, infinite set $\mathcal{F}$. In this case, most of the black-box reduction works just as in the finite case. The non-trivial part is the union bound to obtain high probability bounds. After bounding the probability that a single comparator $f$ has small approximate regret $\apx(f,\eps)$ with probability at least $\delta$, using union bound, we can derive a high-probability bound for all comparators simultaneously with an additional $\log(\abs{F})$ factor. This is not possible for an uncountable infinite class $\mathcal{F}$. To this end one needs tail bounds uniformly over $\mathcal{F}$ of the form: For any $\delta >0$, with probability at least $1 - \delta$,
$$
\sup_{f \in \mathcal{F}}\left\{ \sum_{t=1}^T \tilde{\ell}_t(f) - (1 + \epsilon) \sum_{t=1}^T \ell_t(f) \right\} \le \mathcal{R}_T\left(\mathcal{F}, \epsilon^2/d\right) + \frac{d \log(1/\delta)}{\epsilon^2}
$$
where $\mathcal{R}_T\left(\mathcal{F}, \epsilon^2/d\right)$ is the so called offset Rademacher Complexity introduced in \cite{RakSri15}.
This capacity measure of class $\F$ is defined as:
$$
\mathcal{R}_T(\mathcal{F},\epsilon) = \sup_{\mathbf{x}_{1},\ldots,\mathbf{x}_T}\mathbb{E}_{\sigma}
\left[\sup_{f \in \mathcal{F}}\sum_{t=1}^T \sigma_t f(\mathbf{x}_t(\sigma_{1},\ldots,\sigma_{t-1})) 
- \epsilon  (f(\mathbf{x}_t(\sigma_{1},\ldots,\sigma_{t-1})))^2 \right]
$$
where in the above each $\mathbf{x}_t: \{\pm1\}^{t-1} \mapsto \X$ is 
a mapping from a sequence of $\pm1$ bits to the context space and $\sigma_1,\ldots,\sigma_T$ are Rademacher random variables. This capacity measure is close in spirit to the Rademacher
complexity. If one drops the quadratic term in the definition, this would correspond to the sequential Rademacher complexity \cite{rakhlin2010online}. The quadratic term subtracted makes this complexity measure smaller than the Rademacher complexity. As a specific example, for a finite $\mathcal{F}$, this complexity can be bounded as  $\mathcal{R}_T(\mathcal{F},\epsilon) \le \frac{\log|\mathcal{F}|}{\eps}$ thus giving us the finite class result as a special case.
The above tail bound is proved in the Appendix \ref{ssec:proofs_contextual} and 
is based on the results from \cite{RakSri17}. Using the tail bound above, we prove the following result just as in the finite case.
\begin{theorem}\label{thm:contextual_infinite}
\vspace{0.1in}
Assume that the oracle-efficient full-information algorithm $\mathcal{A}$ has $\eps$-approximate regret $B(\eps,T,\mathcal{F})$ when applied on losses in $[0,1]$. Then our freezing algorithm with learning parameter $\eps'=\frac{\eps}{2}$ on input $\mathcal{A}$ guarantees, for any $\delta>0$, with probability $1-\delta$ $\eps$-approximate regret of
$$
\apx(f,\eps)=\bigO\prn*{\frac{d}{\epsilon} B\prn*{\eps,T,\mathcal{F}} + \mathcal{R}_T\left(\mathcal{F}, \epsilon^2/d\right) + \frac{d \log(1/\delta)}{\epsilon^2}}.
$$ 
\end{theorem}

\subsection{Shifting comparators}\label{ssec:shifting_comparators}
The last application studied in this paper is \emph{shifting bandits} where the comparator is stronger than the  best single fixed arm. In shifting bandits, the set $\mathcal{F}$ corresponds to a sequence of arms in $[d]$ that change over the $T$ rounds, and for $f \in \mathcal{F}$, we use $f(t)$ to denote the arm in the sequence at time $t$. We denote $K(f)=\abs{t<T: f(t)\neq f(t+1)}$ as the number of times $f$ changes over the $T$ rounds.

\textbf{Black-box reduction from full-information shifting algorithms.} With full information and losses lying in $[0,1]$, the expected approximate regret with shifting comparators is $\bigO\prn*{\frac{S(d,K(f),T)}{\epsilon}}$ where $S(d,K(f),T)=K(f) \log(d T) + 2 \log(d)$. This guarantee is, for instance, satisfied by \emph{Noisy Hedge}, i.e., multiplicative weights with uniform noise of $1/T$ added in, as presented in \cite{FosterLLST16}. Our black box reduction leads to bandit shifting algorithms with the following bound on approximate regret. 
\begin{theorem}\label{thm:shifting_black_box}
\vspace{0.1in}
Let $\mathcal{A}$ be any full information algorithm with an expected approximate regret guarantee for shifting comparators of: $\En\brk*{\apx\prn*{f,\eps/5}} \le \frac{5 L\cdot S(d,K(f),T)}{\eps}$,  against any sequence of arms $f\in[d]^T$, when run on losses in $[0,L]$. For any $\delta>0$, the \emph{Double-Threshold Freezing Algorithm} (Algorithm~\ref{alg:black_box})  run with learning parameter $\epsilon'=\eps/5$ on input $\mathcal{A}$, $\alpha$, $d$, has $\epsilon$-approximate regret for shifting comparators of: $\apx\prn*{f,\eps}=O\prn*{\frac{d\cdot \prn*{S\prn*{d,K(f),T}+K(f)\log\prn*{dT/\delta}}}{\eps^2}}$ with probability $1-\delta$.
\end{theorem}
\begin{proof}
The proof follows the arguments in Section~\ref{sec:black_box}. The pseudoregret guarantee copies exactly the proof of Theorem~\ref{thm:black_box} replacing the approximate regret term $A(d,T)$ by its shifting analogue $S(d,K(f),T)$ and noting that the independence number in bandits is equal to the number of arms~$d$.

For the high-probability guarantee, we copy the proof of Theorem~\ref{thm:black_box_whp} with a slight adaptation of item 3 in the proof as we need a more intricate union bound to account for the exponential number of comparators. A vanilla union bound would lead to a linear dependence on time horizon since the number of comparators is exponential. Instead, for comparator $f$, we create an approximate regret with failure probability $\delta'=\frac{\delta}{1+ {T\choose K(f)}}$. This provides a failure probability for each comparator level (number of changes) of $\delta/T$ and then splits it uniformly across comparators of the same number of changes. Therefore, what is coming as a linear term from the $\log\prn*{1/\delta'}$ term is a term logarithmic in $T$ and linear in the number of changes of the comparator instead of the time horizon. Since the linear term already appears in $S(d,K(f),T)$, this does not increase the asymptotic bound on the regret.
\end{proof}

\textbf{Obtaining shifting guarantees directly via freezing.} Interestingly, freezing not only enables transferring existing shifting guarantees from full information to partial information, but can also provide an alternative way to derive such shifting guarantees in the first place. Recall the GREEN-IX algorithm (Algorithm~\ref{alg:green_ix}), the algorithm that provided the optimal high-probability guarantee for bandits. The key property that this algorithm had was that the cumulative estimated losses stayed close to each other. As a result, when the comparator changes, the new comparator is not too penalized by its past performance and can therefore recover pretty fast, leading to effective shifting regret guarantees.
\begin{theorem}\label{thm:green_ix_shifting}
\vspace{0.1in}
 GREEN-IX run with learning parameter $\epsilon'=\epsilon/2$ guarantees an $\epsilon$-approximate regret with shifting comparators of $\bigO\prn*{\frac{(1+K(f))d\log(d/\delta)}{\epsilon}}$  with probability at least $1-\delta$.
\end{theorem}
The proof follows the arguments of Theorem~\ref{thm:green_ix} for each intervals where the comparator is fixed, but uses freezing to create interval-based analogs of two lemmas used in that proof. In particular, the analogue of Lemma~\ref{lem:estimated_losses_close} is:
\begin{lemma}
\vspace{0.1in}
GREEN-IX satisfies the following for any arms $i,j\in[d]$ and any interval $[\tau,\tau']$:
\begin{align*}
\sum_{t=\tau}^{\tau'}\widetilde{\ell_i^t}\leq \sum_{t=\tau}^{\tau'} \widetilde{\ell_j^t}+\frac{2}{\gamma}+\frac{2\ln(1/\gamma)}{\eta}
\end{align*}
\end{lemma}
\begin{proof}
 By Lemma~\ref{lem:estimated_losses_close}, for all arms $i$ and $j$:
\begin{align*}
    \sum_{t=1}^{\tau'-1}\widetilde{\ell}_i^t\leq \sum_{t=1}^{\tau'-1}\widetilde{\ell_j^t}+\frac{1}{\gamma}+\frac{\ln(1/\gamma)}{\eta} \qquad  - \sum_{t=1}^{\tau-1}\widetilde{\ell}_i^t\leq -\sum_{t=1}^{\tau-1}\widetilde{\ell_j^t}+\frac{1}{\gamma}+\frac{\ln(1/\gamma)}{\eta}
\end{align*}
Combining the two, the lemma follows.
\end{proof}
What is left is a shifting analogue for Lemma~\ref{lem:mwu_second_order_regret}:
\begin{lemma}
\vspace{0.1in}
For any interval $[\tau,\tau']$ where the comparator $f$ is fixed, it holds that:
\begin{align*}
\sum_{t=\tau}^{\tau'-1} \sum_i p_i^t \widetilde{\ell_i^t}-\sum_{t=\tau}^{\tau'-1}  \widetilde{\ell}_f^t \leq \eta \sum_{t=\tau}^{\tau'-1}\sum_i p_i^t\prn*{\widetilde{\ell_i^t}}^2+ \frac{\ln(10/\gamma)}{\eta}
\end{align*}
\end{lemma}
\begin{proof}
Following the proof of Theorem 2.4 in \cite{prediction_book}, we define $W_n=\sum_{i\in[d]}\exp(-\eta \sum_{t\leq n}\widetilde{\ell_i^t})$. By the exact same arguments, it holds that:
$$
\ln\prn*{\frac{W_{\tau'}}{W_{\tau}}}\leq  -\eta\cdot \sum_{t=\tau}^{\tau'} \sum_i p_i^t \widetilde{\ell_i^t} +\eta^2\cdot \sum_{t=\tau}^{\tau'} \sum_i p_i^t \prn*{\widetilde{\ell_i^t}}^2 .
$$
To complete the proof, we also need to lower bound this quantity. 
Let $q_i=\frac{\exp(-\eta\sum_{t\leq\tau}\widetilde{\ell_i^t})}{W_{\tau}}$ be the probability associated with arm $i$ prior to the freezing step. Then:
\begin{align*}
\ln\prn*{\frac{W_{\tau'}}{W_{\tau}}}=\ln\prn*{\sum_{i\in[d]}q_i \exp(-\eta \sum_{t=\tau}^{\tau'}\widetilde{\ell_i^t})}
\geq \ln\prn*{q_f \exp(-\eta \sum_{t=\tau}^{\tau'}\widetilde{\ell_f^t})}
\end{align*}
The fact that we apply freezing provides the desired lower bound on $q_f$. In particular, the last time that the estimated loss of arm $f$ is non-zero, its probability is above $\gamma$ and, at this point, it can incur at most $1/\gamma$ estimated loss. As a result, $$q_f\geq \gamma \cdot \exp(-\eta/\gamma)\geq \gamma/10.$$ 
Combining upper and lower bound, it holds that:
\begin{align*}
-\eta\cdot \sum_{t=\tau}^{\tau'} \sum_i p_i^t\widetilde{\ell_i^t}+\eta^2\cdot \sum_{t=\tau}^{\tau'} \sum_i \prn*{p_i^t\widetilde{\ell_i^t}}^2 \geq -\eta\cdot \sum_{t=\tau}^{\tau'} \widetilde{\ell_f^t} +\ln(\gamma/10),
\end{align*}
which concludes the lemma.
\end{proof}

\begin{proof}[Proof of Theorem~\ref{thm:green_ix_shifting}]
The proof copies the one of Theorem~\ref{thm:green_ix} for each interval when the comparator is fixed (using the two above lemmas instead of their non-interval-based analogues) and then summing the bounds for all the $1+K(f)$ comparators.
\end{proof}

\begin{remark} 
In the black-box reduction we assumed that the full information algorithm was Noisy Hedge (or another shifting full-information algorithm) which satisfies the shifting approximate regret in the full-information case. The reason why the small noise there was essential was to establish that, at the time of the switch of comparator, the new comparator arm does not have too low probability. However, our result on GREEN-IX demonstrates that this is also directly offered by freezing without adding such noise. The reason behind this result is that freezing offers an alternative way to ensure that the probability of each arm (prior to freezing) does not become too small. This shows one more property of freezing: achieving directly shifting guarantees for partial information even when applied on full information algorithms without this property. In fact, the resulting bounds have optimal dependence on $\epsilon$. We note that these bounds do not extend to small-loss bounds when we do not know the number of changes in the sequence (as we cannot tune $\epsilon$ appropriately). However, as we discuss in the next paragraph, the approximate regret guarantees are, in fact, what is needed in game-theoretic settings.
\end{remark}

\textbf{Implications to dynamic population games.} 
Our guarantees on low approximate regret have significant implications to repeated game settings where the player set is evolving over time \cite{LykourisST16,FosterLLST16}. In these papers, learning is used 
as a behavioral assumption. The papers consider a \emph{dynamic population} game where, at every round, each player is independently replaced by an adversarially selected player with some turnover probability. This model of evolving games was introduced in \cite{LykourisST16} and is further studied in \cite{FosterLLST16}. These papers show that in a broad class of games, if all players use algorithms to select their strategies that satisfy low approximate regret with shifting comparators then
 the time-average social welfare of the corresponding learning outcomes is approximately efficient even when the turnover probability is large (constant with respect to the number of players and inversely dependent on the $\eps$ of the approximate regret property). This means that, even when a large
 fraction of the population changes at every single round,  players manage to still adapt to the change and guarantee efficient outcomes in most time steps. The results of this paper can be applied to dynamic population games and strengthen the results of \cite{LykourisST16,FosterLLST16} extending it to games with cost and only partial feedback to the players.  Previous work only provided full information algorithms for achieving low approximate regret. Our results strengthen the behavioral assumption showing that low approximate regret with shifting comparators is achievable even at the presence of partial feedback by a small and natural change 
in any full-information
learning algorithms.

\section{Conclusions}\label{sec:conclusions}
We have shown how to obtain small-loss regret guarantees with high probability for general partial information settings in the graph-based feedback model. Our technique captures as special cases important partial-information paradigms such as  contextual bandits and semi-bandits, as well as learning with shifting comparators and bandit feedback.  
For all these settings, we provide a black-box small-loss high-probability guarantee of $\widetilde{\bigO}(
\alpha^{1/3}(L^{\star})^{2/3}),$ where  $\alpha$ can be thought of as an appropriate problem dimension of each paradigm and corresponds to the  independence number of the graph representing the feedback structure. We improve the guarantee to depending only on $\sqrt{L^{\star}}$ for bandits, semi-bandits, as well as fixed feedback graphs.

A number of
important problems related to our work remain open. Our $\sqrt{L^{\star}}$ bound for feedback graphs depends on the partition number, rather than the independence number and only works for fixed graph. Our results assume undirected feedback graphs, while a number of applications have directed feedback. Moreover, our bounds scale with the maximum independence number instead of the average as the corresponding uniform bounds; this is in some sense necessary but there are ways it could be relaxed. Finally, our results for shifting comparators (in Section \ref{ssec:shifting_comparators}) either require knowing the number of changes of the comparator, or are suboptimal in the dependence on this number. We elaborate on each:
\begin{itemize}
\setlength{\itemsep}{0pt}\setlength{\parsep}{0pt}\setlength{\parskip}{0pt}
    \item \emph{\textbf{Optimal dependence on $L^{\star}$ for general graphs:}} The first question is to derive an algorithm with an optimal dependence of $\widetilde{\bigO}\prn{\sqrt{\alpha L^{\star}}}$  or at least extend our $\widetilde{\bigO}\prn{\sqrt{\kappa L^{\star}}}$ result to evolving graphs. We have
    shown that our framework applied on specific algorithms can lead to such an improvement for the bandit and semi-bandit settings, resolving open questions by \cite{Neu15_semibandits,Neu2015_implicit}. 
    In Section \ref{ssec:clique_partition}
    , we show that when the feedback graph $G$ is not evolving, we obtain a $\widetilde{O}\prn{\sqrt{\kappa (G) L^{\star}}}$ regret guarantee, where $\kappa (G)$ is the minimum clique partition size. Recent subsequent
    work of \cite{AllenZBL2018} gave an optimal pseudo-regret bound for the contextual bandit setting (resolving an open problem raised by \cite{pmlr-v65-agarwal17a}), a special case of graph-based feedback where graphs are evolving and consist of disjoint cliques, which implies that $\kappa(G^t)=\alpha(G^t)$ for all $t$. This provides hope that our work can be extended to a $\widetilde{\bigO}\prn{\sqrt{\alpha L^{\star}}}$ or $\widetilde{\bigO}\prn{\sqrt{\kappa L^{\star}}}$ result for general graphs. 
    \item \emph{{\textbf{Small-loss shifting/tracking regret guarantees:}}} The second question is to derive an algorithm that achieves a shifting regret bound for the partial information case that is $\widetilde{\bigO}\prn{\sqrt{K(f)\cdot L^{\star}}}$ without knowledge of the number of changes $K(f)$. In Section \ref{ssec:shifting_comparators}, we provide an optimal approximate regret bound. Such a bound is directly    useful, for instance, in game-theoretic settings where the approximate regret is the essential quantity. However, unlike all of our other bounds, this one does not lift to small-loss guarantees through the usual doubling trick unless we know the number of changes. This is due to the fact that, if the tuning of $\eps$ does not depend on the number of changes, using the doubling trick only provides a regret bound of $\widetilde{\bigO}\prn{K(f)\sqrt{L^{\star}}}$. In contrast, 
    in the full-information setting,
    $\widetilde{\bigO}\prn{\sqrt{K(f)L^{\star}}}$ guarantee algorithms are known \cite{Daniely15, Luo2015}. It would be therefore interesting to obtain such guarantees in the partial-information setting. We note that the stronger requirement of \emph{strongly adaptive regret} (small regret within each subinterval) is not achievable \cite{Daniely15}.
    \item \emph{\textbf{Maximum vs average independence number:}} Another interesting direction is whether one can derive bounds that scale with the average independence number of the graph instead of the maximum. The trivial way to state such a guarantee would be to ask for a bound that scales with $avg_t(\alpha(G^t))L^{\star}$ instead of $\max_t(\alpha(G^t))L^{\star}$;  such a guarantee is achievable for uniform bounds in graph-based feedback, e.g. \cite{AlonCGMMS}. Unfortunately, for small-loss bounds, this is not attainable. Concretely, assume that there is an algorithm providing $\sqrt{avg_t(\alpha(G^t))L^{\star}}$. Taking the instance that gives the lower bound of $\sqrt{dT}$ for the multi-armed bandit setting and appending it with $(d-1)\cdot T$ rounds where all the arms have $0$ loss and the feedback is full information, we obtain an instance with average independence number of $2$ and the same losses as in the bandit setting. Therefore, if such a bound existed, it would contradict the lower bound for the bandit setting. Despite this negative statement, it is possible that a more modular guarantee can be achieved. More concretely, let $i^{\star}$ be the comparator arm. A bound that would scale with the average of a function of the independence number at the round and the loss of $i^{\star}$ is not precluded from the aforementioned negative statement. 
    \item \emph{\textbf{Directed graphs:}} Finally, it would be nice to extend the graph-based feedback results to directed graphs.  Our work relies on the undirected nature of the graph in controlling the cascade of propagation in the freezing process. Some of the previous work on uniform bounds \cite{AlonCGMMS,AlonCBDK15} also offer bounds that apply also for the directed case.  Providing small-loss bounds for directed feedback graphs is an interesting open problem.
\end{itemize}

\textbf{Acknowledgements} We thank Haipeng Luo for pointing out the connection to contextual bandits, Jake Abernethy, Adam Kalai, and Santosh Vempala for a useful discussion about Follow the Perturbed Leader, and Yang Liu and Mark Sellke for pointing to the fact that the general graph-feedback result can be also obtained without two thresholds. Finally, our results and their presentation have been significantly improved by the excellent feedback from  anonymous reviewers. Part of this work was done while the authors were visiting the Simons Institute.

\bibliographystyle{alpha}
\bibliography{bib1}

\appendix

\section{Supplementary material for Section \ref{sec:black_box}}
\label{app:proofs_section_3}
\subsection{Concentration inequality (Lemma \ref{lem:black_box_high_probability})}
\begin{proof}[Proof of Lemma \ref{lem:black_box_high_probability}]
The proof follows the basic outline of classical Chernoff bounds for independent variables combined with the law of total expectation to handle the dependence.

We start with the first claim. For any real number $b$ to be instantiated later:
\begin{equation}\label{eq:martingale_high_prob}
\Pr\brk*{X-(1+\epsilon)M>b}\le e^{-\lambda b} \En\brk*{e^{\lambda\prn*{X-\prn*{1+\epsilon}M}}}=e^{-\lambda b} \En\brk*{\prod_{t=1}^T e^{\lambda\prn*{x_t-\prn*{1+\epsilon}m_t}}}
\end{equation}
We will prove by induction on $T$ that the expectation above is at most 1 if we use  $\lambda=\ln(1+\epsilon)$. Given this fact, we can set $b$ such that $e^{-\lambda b}=e^{-\ln(1+\epsilon) b}=\delta$.
Using that $\ln(1+\eps)\geq \frac{\eps}{ 1+\eps}$ for all $\eps\geq 0$, it follows that $b=
\frac{\ln\prn*{\frac{1}{\delta}}}{\ln\prn*{1+\eps}}\leq \frac{\prn*{1+\eps}\cdot\ln\prn*{\frac{1}{\delta}}}{\eps}$.

Now consider the expectation $\En\brk*{\prod_t e^{\lambda\prn*{x_t-\prn*{1+\epsilon}m_t}}}$, we will prove by induction on $T$ that with the above choice of $\lambda$ this is at most 1. For the base case of $T=1$ we have a single random variable $x\in[0,1]$ and its expectation $m=\En\brk*{x}$. The expectation is $\En\brk*{e^{\lambda\prn*{x-\prn*{1+\epsilon}m}}}=\En\brk*{e^{\lambda x}}\cdot e^{-\lambda \prn*{1+\epsilon}m}$.

We note that for a value $x\in[0,1]$ and any $\lambda$, the following simple inequality holds:
$$e^{\lambda x} \le x e^{\lambda}   -x+1$$
This is true as it holds with equality for $x=0$ and 1, and the difference is a concave function (as the second derivative of $g(x)=e^{\lambda x} - x e^{\lambda} +x-1$ is $g''(x)=\lambda^2 e^{\lambda x}\ge 0$),  so the inequality is true between the two points.

Now write the expectation as
$$
\En\brk*{e^{\lambda x}}\le \En\brk*{x e^{\lambda}   - x+1}=\En\brk*{x \cdot\prn*{ e^{\lambda}-1}+1}
=m\cdot \prn*{e^{\lambda}-1}+1\le e^{m\cdot \prn*{e^{\lambda}-1}}.
$$
Using this in the expectation we get
\begin{align*}
\En\brk*{e^{\lambda\prn*{x-\prn*{1+\epsilon}m}}}&\le e^{m\cdot \prn*{e^{\lambda}-1}}
\cdot e^{-\lambda\prn*{1+\epsilon}m}
=e^{m\prn*{e^{\lambda}-1-\lambda\prn*{1+\eps}}} \le 1
\end{align*}
where the last inequality follows from the choice of $\lambda=\ln(1+\eps)$, as the multiplier of $m$ in the exponent with this choice of $\lambda$ is 
$$
e^{\lambda}-1-\lambda\prn*{1+\eps}=\eps-\prn*{1+\eps}\ln\prn*{1+\eps}\le \eps-\prn*{1+\eps}\prn*{\eps-\frac{\eps^2}{2}}=-\frac{\eps^2\prn*{1-\eps}}{2}<0.
$$

Now we are ready to prove the general case. Using the law of total expectation, we obtain:
\begin{align*}
\En\brk*{\prod_{t=1}^T e^{\lambda\prn*{x_t-\prn*{1+\epsilon}m_t}}}
&=\En\brk*{\prod_{t=1}^{T-1} e^{\lambda\prn*{x_t-\prn*{1+\epsilon}m_t}}\cdot e^{\lambda\prn*{x_T-\prn*{1+\epsilon}m_T}}}\\
&=\En\brk*{\prod_{t=1}^{T-1} e^{\lambda\prn*{x_t-\prn*{1+\epsilon}m_t}}\cdot \En_{T-1}\brk*{e^{\lambda\prn*{x_T-\prn*{1+\epsilon}m_T}}}}
\end{align*}
where $\En_{t-1}\brk*{\cdot}$ is the random variable taking expectation over the last term conditioned on all the previous terms $x_1,\ldots, x_{T-1}$. Note that conditioned on the previous terms, the conditional expectation $\En_{T-1}\brk*{e^{\lambda\prn*{x_T-\prn*{1+\epsilon}m_T}}}$ is exactly the base case, and hence at most 1 by the above, so we can conclude that 
\begin{align*}
\En\brk*{\prod_{t=1}^T e^{\lambda\prn*{x_t-\prn*{1+\epsilon}m_t}}}
&\le \En\brk*{\prod_{t=1}^{T-1} e^{\lambda\prn*{x_t-\prn*{1+\epsilon}m_t}}}
\end{align*}
and the statement follows by the induction hypothesis. 

To prove the lower bound, we proceed in an analogous way. For $\lambda=-\ln\prn*{1-\eps}$, using that $\frac{1}{1-\eps}\geq 1+\eps$, we obtain the equivalent of the inequality \eqref{eq:martingale_high_prob} with 
$b'=
\frac{\ln\prn*{1/\delta}}{\ln\prn*{\frac{1}{1-\eps}}}\leq \frac{\ln\prn*{1/\delta}}{\ln\prn*{1+\eps}}$.
\begin{equation}
\Pr\brk*{\prn*{1-\eps}M-X>b'}\le e^{-\lambda b'} \En\brk*{e^{\lambda\prn*{\prn*{1-\eps}M-X}}}=e^{-\lambda b'} \En\brk*{\prod_{t=1}^T e^{\lambda\prn*{\prn*{1-\epsilon}m_t-x_t}}}
\end{equation}
Regarding the bound on the expectation, consider
first a single variable $m=\En\brk*{x}$.
$$\En\brk*{e^{-\lambda x}} \le \En\brk*{x e^{-\lambda}-x+1}
=m\prn*{e^{-\lambda}-1}+1\le e^{m\prn*{e^{-\lambda}-1}}
$$
We now bound the expectation as
$$
\En\brk*{e^{\lambda\prn*{\prn*{1-\eps}m-x}}}\le e^{\lambda (1-\eps)m}\En\brk*{e^{-\lambda x}}\le e^{\lambda \prn*{1-\eps}m}\cdot e^{m\cdot \prn*{e^{-\lambda}-1}}
=e^{m(\lambda(1-\eps)+(e^{-\lambda}-1))}\le 1
$$
where the last inequality follows from the choice of $\lambda=-\ln(1-\eps)$, as the multiplier of $m$ in the exponent with this choice of $\lambda$ is 
$$
\lambda(1-\eps)+(e^{-\lambda}-1)=-(1-\eps)\ln(1-\eps)-\eps \le (1-\eps)\eps -\eps=-\eps^2 <0.
$$
using the fact that $\ln(1-\eps) \le -\eps$. The induction then follows as before.
\end{proof}

\subsection{Transforming approximate regret to small-loss guarantees (Lemma \ref{lem:double})}
\begin{proof}[Proof of Lemma~\ref{lem:double}]
We denote the loss of the algorithm within phase $\tau$ as $\widehat{L}_\tau$ and the loss of the best arm within the phase as $L_\tau^{\star}$. Now note that on any phase $\tau$, by our premise about approximate regret on each phase, we have that with probability at least $1 - \delta'$,
$$
\widehat{L}^\tau - L^\star_\tau \le \epsilon_\tau \widehat{L}_\tau + \frac{\Psi(\delta')}{(\epsilon_\tau)^q}
$$
The term $\epsilon_\tau \widehat{L}_\tau$ of the right hand side can be split in two terms, i) all but the last round of the phase and ii) the last round. The first term is bounded  by $ \frac{\Psi\prn*{\delta'}}{\prn*{\eps_\tau}^q}$ due to the doubling condition. The second term can be upper bounded by $\eps_{\tau}$ since the losses are in $[0,1]$. Hence, for phase $\tau$, with probability $1 - \delta'$:
$$
\widehat{L}_\tau-L_\tau^{\star}\leq \frac{2\cdot \Psi\prn*{\delta'}}{\prn*{\eps_\tau}^q}+\eps_\tau.
$$
Letting $\Gamma$ denote the last phase and summing over the phases, we have:
\begin{align*}
 \widehat{L}-L^{\star}&\leq \sum_{\tau=0}^{\Gamma-1} \frac{2\Psi\prn*{\delta'}}{\prn*{\eps_\tau}^q}+\sum_{\tau=0}^{\Gamma-1} \eps_\tau +\eps_\Gamma \widehat{L}_\Gamma + \frac{\Psi\prn*{\delta'}}{\prn*{\eps_\Gamma}^q}\\
&\leq 2\Psi\prn*{\delta'} \sum_{\tau=0}^\Gamma \frac{1}{\kappa}^{-q\cdot \tau}+\sum_{\tau=0}^{\Gamma-1}2^{-\tau}+\eps_\Gamma \widehat{L}_\Gamma\\
&\leq 2\Psi\prn*{\delta'} \cdot\frac{\kappa^{q(\Gamma+1)}-1}{\kappa^q-1}+\frac{\kappa}{\kappa-1}+\eps_\Gamma \widehat{L}_\Gamma\\
&\leq \frac{2\kappa^q}{\kappa-1} \Psi\prn*{\delta'}\frac{1}{\prn*{\eps_\Gamma}^q}+\frac{\kappa}{\kappa-1}+\eps_\Gamma \widehat{L}_\Gamma & \textrm{Since $q \ge 1$}\\
&\leq \prn*{\frac{\frac{2\kappa^q}{\kappa-1} \Psi\prn*{\delta'}}{\prn*{\eps_\Gamma}^q}}^{\frac{1}{q+1}}\cdot\prn*{\kappa^q \cdot\frac{\frac{2\kappa^q}{\kappa-1} \Psi\prn*{\delta'}}{\prn*{\eps_{\Gamma-1}}^q}}^{\frac{q}{q+1}}+\prn*{\eps_\Gamma \widehat{L}_\Gamma}^{\frac{1}{q+1}}\cdot \prn*{\eps_\Gamma \widehat{L}_\Gamma}^{\frac{q}{q+1}}+\frac{\kappa}{\kappa-1}\\
&\leq \prn*{\frac{\frac{2\kappa^q}{\kappa-1}\Psi\prn*{\delta'}}{\prn*{\eps_\Gamma}^q}}^{\frac{1}{q+1}}\cdot \prn*{\frac{2\kappa^{2q}}{\kappa-1}\eps_{\Gamma-1}\widehat{L}_{\Gamma-1}}^{\frac{q}{q+1}}+\prn*{\frac{\Psi\prn*{\delta'}}{\prn*{\eps_{\Gamma}}^q}}^{\frac{1}{q+1}}\cdot \prn*{\eps_{\Gamma}\widehat{L}_\Gamma}^{\frac{q}{q+1}}+\frac{\kappa}{\kappa-1}\\
&\leq \prn*{\frac{\frac{2\kappa^q}{\kappa-1}\Psi\prn*{\delta'}}{\prn*{\eps_\Gamma}^q}}^{\frac{1}{q+1}}\cdot \prn*{\frac{2\kappa^{2q+1}}{\kappa-1} \eps_{\Gamma}\widehat{L}_{\Gamma-1}}^{\frac{q}{q+1}}+\prn*{\frac{\Psi\prn*{\delta'}}{\prn*{\eps_{\Gamma}}^q}}^{\frac{1}{q+1}}\cdot \prn*{\eps_{\Gamma}\widehat{L}_\Gamma}^{\frac{q}{q+1}}+\frac{\kappa}{\kappa-1}\\
\end{align*}
Upper bounding $\widehat{L}_{\Gamma-1}\leq \widehat{L}$ and $\widehat{L}_{\Gamma}\leq \widehat{L}$, we conclude that:
\begin{align}\label{eq:from_lhat_to_lst}
\widehat{L}-(L^{\star}+\frac{\kappa}{\kappa-1})\leq\prn*{\frac{4\kappa^{2q^2+2q+1}+1}{(\kappa-1)^{q+1}}\cdot \Psi\prn*{\delta'}}^{\frac{1}{q+1}}\cdot \prn*{\widehat{L}}^{\frac{q}{q+1}}
\end{align}
To replace the dependence of $\widehat{L}$ by $L^{\star}$, we apply Young's inequality, the approximate regret property, and the sub-additivity property.
\begin{align*}
\widehat{L}-(L^{\star}+ +\frac{\kappa}{\kappa-1})&\leq\prn*{\frac{4\kappa^{2q^2+2q+1}+1}{(\kappa-1)^{q+1}}\cdot \Psi\prn*{\delta'}}^{\frac{1}{q+1}}\cdot \prn*{\widehat{L}}^{\frac{q}{q+1}}\leq  \frac{1}{q+1}\cdot\frac{4\kappa^{2q^2+2q+1}+1}{(\kappa-1)^{q+1}}
\cdot \Psi\prn*{\delta'}+\frac{q}{q+1}\widehat{L}\Rightarrow
\end{align*}

\begin{align*}
    \widehat{L}\leq (q+1)\cdot(L^{\star}+\frac{\kappa}{\kappa-1})+\prn*{\frac{4\kappa^{2q^2+2q+1}+1}{(\kappa-1)^{q+1}}\cdot\Psi\prn*{\delta'}}
\end{align*}
Applying this to \eqref{eq:from_lhat_to_lst} we obtain:
\begin{align*}
    \widehat{L}-L^{\star}
    &\leq\prn*{\frac{4\kappa^{2q^2+2q+1}+1}{(\kappa-1)^{q+1}}\cdot \Psi\prn*{\delta'}}^{\frac{1}{q+1}}\cdot \prn*{(q+1)\cdot(L^{\star}+\frac{\kappa}{\kappa-1})+ \frac{4\kappa^{2q^2+2q+1}+1}{(\kappa-1)^{q+1}\cdot\Psi\prn*{\delta'}}}^{\frac{q}{q+1}}+\frac{\kappa}{\kappa-1}
\end{align*}
Since there are at most $\log\prn*{L^{\star}+1}+1$ phases, setting $\delta'=\frac{\delta}{\log\prn*{L^{\star}+1}+1}$ suffices for the high probability statements to hold for all phases.
\end{proof}

\section{Supplementary material for Section \ref{sec:applications}}
\label{app:sec:contextual}
\subsection{Sample complexity to estimate probabilities in oracle-efficient settings}
\label{app:sampling}
In this section we provide a formal upper bound on the number of samples needed to create the estimates on the probabilities we use for semi-bandits (Section \ref{ssec:semi_bandits_black_box}) and for oracle-efficient contextual bandits (Section \ref{ssec:contextual_bandits_black_box}). Since we are only allowed to sample solutions instead of directly computing the observation probabilities $P_e^t$,
we draw $N$ samples at each round and create estimates on these probabilities.

Let $f^t_1,\ldots,f^t_N$ be these samples, and
and let $\widehat{P}^t_e = \frac{1}{N}\sum_{i=1}^N \mathbf{1}_{e \in f^t_i}$ be the empirical frequency of appearances of each element $e$. 
We now run our algorithm by doing importance sampling using $\widehat P$. 
We use the losses $\widehat{\ell}_e^t=\frac{\ell_e^t}{\widehat{P}_e^t}$ as estimated losses for the elements $e$ of the selected solution
and apply freezing also based on this estimate $\hat P$. That is we freeze element $e$ if $\widehat{P}^t_e < \gamma$ and only look for solutions that do not contain frozen edges while sampling solutions. This gives us the following lemma that, with the appropriate concentration can result to an approximate regret guarantee as in Theorems \ref{thm:semisub} and \ref{thm:optimal_semi_bandits}.

\begin{lemma}\vspace{0.1in}
If we run the semi-bandit algorithm with $\epsilon'=\epsilon/6$ based on $\widehat{P}^t_e$'s as estimates in place of  $P^t_e$'s, and $N = \frac{(1+2\epsilon')m \log(d T/\delta)}{\epsilon' \gamma} $ then for any $\epsilon , \delta>0$ with probability at least $1 - \delta$ over randomization, the following inequality is true:
\begin{align*}
    \sum_{t} \ell^t_{I(t)} & \le    (1 + \epsilon)  \sum_{t} \sum_{e \in f} \tilde{\ell}^t_{e} +  \frac{2}{\gamma} + \frac{m B(\epsilon/6,T, \mathcal{F})}{\gamma}
\end{align*}
\end{lemma}
\begin{proof}
From Lemma \ref{lem:black_box_high_probability},
 with probability at least $1 - \delta$, both the following statements are true:
\begin{align}\label{eq:one_conc}
    \forall e \in \mathcal{E}, t\leq T ~~~  \widehat{P}^t_e \le (1 + \epsilon)   P^t_e + \frac{(1+\epsilon)\log(dT/\delta)}{N \epsilon} \text{.  and}
\end{align}
\begin{align}\label{eq:two_conc}
\forall e \in \mathcal{E}, t\leq T ~~~ 
P^t_e  \le (1 + \epsilon) \widehat{P}^t_e + \frac{(1+2\epsilon)\log(dT/\delta)}{N \epsilon}.
\end{align}

We adapt the analysis in the proof of Theorem \ref{thm:semisub} to deal with the fact that we apply importance sampling based on samples and not the actual probabilities:
\begin{align*}
\sum_{t} \ell^t_{I(t)} &= \sum_{t} \sum_{e \in I(t)} \ell^t_{e}  = \sum_{t} \sum_{e \in I(t)} W^t_e \widetilde{\ell}^t_{e}  \\
&= \sum_t \sum_{e\in I(t)} W_e^t\prn*{\widetilde{\ell}_e^t-(1+\eps')\widehat{\ell}_e^t}+(1+\eps')\sum_t \sum_{e\in I(t)}W_e^t\widehat{\ell}_e^t\\
&= \sum_t \sum_{e\in I(t)} W_e^t\prn*{\widetilde{\ell}_e^t-(1+\eps')\widehat{\ell}_e^t}+(1+\eps')\sum_t \sum_{e\in \mathcal{E}}W_e^t\widehat{\ell}_e^t\\ 
&\leq \sum_t \sum_{e\in I(t)} W_e^t\prn*{\widetilde{\ell}_e^t-(1+\eps')\widehat{\ell}_e^t}+\frac{1+\eps'}{1-\eps'}\sum_t \sum_{e\in \mathcal{E}}P_e^t\widehat{\ell}_e^t\\
&= \sum_t \sum_{e\in I(t)} W_e^t\prn*{\widetilde{\ell}_e^t-(1+\eps')\widehat{\ell}_e^t}+\frac{1+\eps'}{1-\eps'}\sum_t \sum_{i\in \mathcal{F}}p_i^t\widehat{\ell}_i^t\\
&\leq \sum_t \sum_{e\in I(t)} W_e^t\prn*{\widetilde{\ell}_e^t-(1+\eps')\widehat{\ell}_e^t}+\frac{1+\eps'}{(1-\eps')^2}\prn*{\sum_t \sum_{e\in f}\widehat{\ell}_e^t+\frac{m B(\eps',T,\mathcal{F})}{\gamma}}\\ 
&= \sum_{t} \sum_{e \in I(t)} W^t_e \prn*{\widetilde{\ell}^t_{e} - (1+\epsilon') \widehat{\ell}^t_{e}} +  \frac{1 + \epsilon'}{(1 - \epsilon')^2}  \sum_t  \sum_{e \in f} \prn*{\widehat{\ell}^t_e - (1 + \epsilon')\widetilde{\ell}^t_e} \\
& ~~~~~~~~~~ +  \frac{(1 + \epsilon')^2}{(1 - \epsilon')^2} \left( \sum_{t} \sum_{e \in f} \widetilde{\ell}^t_{e}  + \frac{m B(\epsilon',T, \mathcal{F})}{\gamma}\right)
\end{align*}
where we used Lemma \ref{lem:freezing_semi_bandits} for the first inequality and the regret bound for the full information algorithm for the second one, noting that the magnitude of the losses is $L=\frac{m}{\gamma}$. The third term can be bound similarly as in the proof of Theorem \ref{thm:semisub}. What is left is to bound the two first terms.

\textbf{First term:} With probability at least $1-\delta$, the following holds:
\begin{align*}
 \sum_{t} \sum_{e \in I(t)} W^t_e \prn*{\widetilde{\ell}^t_{e} - (1+\epsilon') \widehat{\ell}^t_{e}}&=  \sum_{t} \sum_{e \in I(t)} W^t_e \ell^t_{e} \prn*{\frac{1}{P^t_e} - (1+\epsilon') \frac{1}{\widehat{P}^t_e}}\\
&=  \sum_{t} \sum_{e \in I(t)} W^t_e \ell^t_{e} \frac{1}{\widehat{P}^t_e P^t_e}\prn*{\widehat{P}^t_e - (1+\epsilon') P^t_e}\\
&\leq \sum_{t} \sum_{e \in I(t)} W^t_e \ell^t_{e} \frac{1}{\widehat{P}^t_e P^t_e}\frac{(1+\eps')\log(dT/\delta)}{N\eps'}\\
&\leq   \frac{(1+\eps')\log(dT/\delta)}{N (1-\eps') \eps'} \sum_{t} \sum_{e \in I(t)} \frac{\ell^t_e}{\widehat{P}^t_e}\\
&\le  \frac{(1+\eps')\log(dT/\delta)}{N \gamma (1-\eps') \eps'} \sum_{t} \sum_{e \in I(t)} \ell^t_e\\
&=  \frac{(1+\eps')\log(dT/\delta)}{N \gamma (1-\eps') \eps'} \sum_{t}  \ell^t_{I(t)}\\
& \leq \eps' \sum_{t}  \ell^t_{I(t)}
\end{align*}
where the first inequality holds with probability $1-\delta$ by relation \eqref{eq:one_conc}. The second inequality holds by Lemma \ref{lem:freezing_semi_bandits}. The third inequality holds because $\widehat{P}_e^t\geq \gamma$ for non-frozen nodes. The last inequality holds since $N\geq \frac{(1+\eps')\log(dT/\delta)}{(\eps')^2 \gamma (1-\eps')}$.

\textbf{Second term:} We focus only on rounds where $f$ is not frozen as else the quantity is anyway negative. With probability at least $1-\delta$, the following holds:
\begin{align*}
    \frac{1 + \epsilon'}{(1 - \epsilon')^2}  \sum_t \sum_{e \in f} \left(\widehat{\ell}^t_e - (1 + \epsilon')\widetilde{\ell}^t_e\right) &=  \frac{1 + \epsilon'}{(1 - \epsilon')^2}  \sum_t \sum_{e \in f} \ell^t_e \left(\frac{1}{\widehat{P}^t_e} - (1 + \epsilon')\frac{1}{P^t_e}\right)\\
    &= \frac{1 + \epsilon'}{(1 - \epsilon')^2}  \sum_t \sum_{e \in f} \frac{\ell_e^t}{\widehat{P}^t_e P^t_e} \prn*{P^t_e - (1 + \epsilon') \widehat{P}^t_e}\\
     &\leq \frac{1 +\epsilon'}{(1 -\epsilon')^2}  \sum_t \sum_{e \in f} \frac{\ell_e^t}{\widehat{P}^t_e P^t_e} \cdot\frac{(1+2\eps')\log(dT/\delta)}{N\eps'}\\
      &\leq \frac{1 +\epsilon'}{(1 -\epsilon')^2} \cdot\frac{(1+2\eps')\log(dT/\delta)}{N\eps'} \cdot  \sum_t \sum_{e \in f} \frac{\widetilde{\ell}_e^t}{\widehat{P}^t_e} \\
&\leq \frac{1 +\epsilon'}{(1 -\epsilon')^2} \cdot\frac{(1+2\eps')\log(dT/\delta)}{N \gamma \eps'} \cdot  \sum_t \sum_{e \in f} \widetilde{\ell}_e^t \\
&\leq \eps' \sum_t \sum_{e \in f} \widetilde{\ell}_e^t = \eps' \sum_t \sum_{e \in f} \widetilde{\ell}_e^t
\end{align*}
The first inequality holds from relation \eqref{eq:two_conc}. The second inequality holds since $\widetilde{\ell}^t_e = \ell^t_e/P^t_e$.  The third inequality holds because for non-frtozen arms $\widehat{P}^t_e > \gamma$. The final inequality  hold as $N\geq \frac{2(1+2\eps')\log(dT/\delta)}{(\eps')^2 \gamma}\cdot\frac{(1+\eps')^2}{(1-\eps')^2}$.

\textbf{Summing up:} Using the above two bounds on the summands we finally conclude that using $N = \frac{2(1+2\eps')\log(dT/\delta)}{(\eps')^2 \gamma}\cdot\frac{(1+\eps')^2}{(1-\eps')^2}$ samples on every round, 
with probability at least $1 - \delta$,
\begin{align*}
\sum_{t} \ell^t_{I(t)} & \le
\eps'\sum_t\ell_{I(t)}^t+\eps'\sum_t \sum_{e\in f} \widetilde{\ell}_e^t+
\frac{(1+\eps')^2}{(1-\eps')^2}\prn*{ \sum_{t} \sum_{e \in f} \widetilde{\ell}^t_{e}  + \frac{m B(\epsilon',T, \mathcal{F})}{\gamma}}
\end{align*}
Assuming $\epsilon < 1$ and $\epsilon'=\epsilon/6$, we can conclude that 
\begin{align*}
\sum_{t} \ell^t_{I(t)} & \le    (1 + \epsilon)  \sum_{t} \sum_{e \in f} \widetilde{\ell}^t_{e} + \frac{(1+\frac{\eps}{6})^2m B(\epsilon/6,T, \mathcal{F})}{(1-\frac{\eps}{6})^2\gamma}.
\end{align*}
\end{proof}
The above lemma can be applied
in Section \ref{ssec:semi_bandits_black_box}, e.g., in the proof of Theorem \ref{thm:optimal_semi_bandits} to replace analysis for the case when $P^t_e$'s can be computed exactly to one where they aare estimated via sampling. The statements are still true in high probability with only additional $\log(T)$ factor term in regret bound. Moreover, it can be also applied in Section~\ref{ssec:contextual_bandits_black_box} for the oracle-efficient guarantees we provide.

\subsection{Proofs for contextual bandits (Theorems \ref{thm:contextual} and \ref{thm:contextual_infinite})}\label{ssec:proofs_contextual}
\begin{proof}[Proof of Theorem~\ref{thm:contextual}]
The proof is similar to the proof of Theorem \ref{thm:semisub} adjusted to the contextual bandit setting. Note that the estimated losses $\widetilde{\ell}^t$ are bounded by $1/\gamma$ and so, by our premise about the full information algorithm, we have:
\begin{align}
   (1 - \eps') \sum_{t=1}^T\ell^t_{I(t)} & = (1 - \eps') \sum_{t=1}^T \sum_{i=1}^d w^t_i \tilde{\ell}^t_i \notag\\
    & \le \min_{f \in \mathcal{F}}\sum_{t=1}^T \tilde{\ell}^t_{f(x_t)} + \frac{1}{\gamma}  B(\epsilon',T,\mathcal{F}) \label{eq:base}
\end{align}
Now note that by Lemma \ref{lem:black_box_high_probability} for any $f \in \mathcal{F}$, and any $\delta > 0$, with probability at least $1 - \delta$, 
$$
\sum_{t=1}^T \tilde{\ell}^t_{f(x_t)} 
\le (1 + \epsilon')\sum_{t=1}^T \tilde{\ell}^t_{f(x_t)} 
+ \frac{(1 + \eps') \log(1/\delta)}{\eps'}
$$
Hence using union bound over $f \in \mathcal{F}$ we conclude that for any $\delta > 0$, with probability at least $1 - \delta$,
\begin{align}\label{eq:union}
   \min_{f \in \mathcal{F}} \sum_{t=1}^T \tilde{\ell}^t_{f(x_t)} \le (1 + \epsilon')\min_{f \in \mathcal{F}} \sum_{t=1}^T \tilde{\ell}^t_{f(x_t)} + \frac{(1 + \eps') \log(|\mathcal{F}|/\delta)}{\eps'}
\end{align}
Plugging the above in Eq. \ref{eq:base} we conclude that for any $\delta$, with probability at least $1- \delta$, 
\begin{align*}
    (1 - \eps') \sum_{t=1}^T\ell^t_{I(t)} &\le (1 + \eps')\min_{f \in \mathcal{F}}\sum_{t=1}^T \tilde{\ell}^t_{f(x_t)} + \frac{(1 + \eps') \log(|\mathcal{F}|/\delta)}{\eps'} + \frac{1}{\gamma}  B(\epsilon',T,\mathcal{F})
\end{align*}
Since $\frac{1+\eps'}{1 - \eps'} = \frac{1}{1-\eps}$ for $\eps'=\frac{\eps}{2}$ and using the fact that $\gamma = \eps'/d$, we conclude the proof.
\end{proof}

\begin{proof}[Proof of Theorem \ref{thm:contextual_infinite}]
The analogue of Theorem \ref{thm:contextual} for infinite set $\mathcal{F}$ has a similar proof as above. In fact, notice that if one has a full information algorithm for the problem in the same vein as in the above theorem, then irrespective of the fact $\mathcal{F}$ could be an uncountably infinite class, the proof till Eq. \ref{eq:base} is just true. The main hurdle comes when proving the analogue of Eq. \ref{eq:union}. Unlike the finite $\mathcal{F}$ case, where we simply used union bound over the concentration statement from Lemma \ref{lem:black_box_high_probability} for infinite classes we need a more careful analysis. Specifically we will replace Eq. \ref{eq:base} by an appropriate tail bound given by: for any $\delta > 0$, w.p. at  least $1- \delta$:
$$
\sup_{f \in \F} \sum_t \left\{  \tilde{\ell}^t_{f(x_t)} - (1 + \epsilon) \ell^t_{f(x_t)}  \right\} \le \mathcal{R}_T(\mathcal{F},\gamma \epsilon) +  \frac{\log(3/\delta)}{\epsilon \gamma}
$$
where the term $\mathcal{R}_T(\mathcal{F},\epsilon)$ is the so called offset Rademacher complexity of class $\mathcal{F}$ defined in \cite{RakSri15}. 

The above tail bound is proved in the following lemma. Using this tail bound in place of Eq. \eqref{eq:union} and plugging it into Eq. \eqref{eq:base} yields the infinite comparator version of Theorem \ref{thm:contextual}.
\end{proof}

Now we are ready to prove the tail bound for the infinite class which is based on result from \cite{RakSri17}. 
\begin{lemma}\vspace{0.1in}
For any possibly infinite class $\mathcal{F}$, (under  mild assumption of $\F$ and $\X$ for measurability), for any $\delta > 0$, w.p. at  least $1- \delta$:
$$
\sup_{f \in \F} \sum_{t=1}^T \left\{   (1 -\eps)\tilde{\ell}^t_{f(x_t)} - (1 + \eps) \ell^t_{f(x_t)})  \right\} >  \mathcal{R}_T(\F,\epsilon \gamma)+ \frac{ \log(3/\delta)}{\gamma \epsilon}
$$
\end{lemma}
\begin{proof}
 Let us define the random variable $Z_t = (x_t,\tilde{\ell}^t)$. and define $\mathbb{E}_{t-1}[\cdot] = \mathbb{E}\left[ \cdot | x_1,\ldots,x_t,  \tilde{\ell}^1,\ldots,\tilde{\ell}^{t-1} \right]$. Further we define the class $\G$ such that  each $g \in \G$ corresponds to an $f \in \F$  and $g(Z_t) = \gamma  \tilde{\ell}^t_{f(x_t)}$. Notice that $|g(Z_t)| \le 1$ because losses are bounded by $1/\gamma$. Now, using Corollary 8 in \cite{RakSri17} with regret bound for online non-parametric regression from \cite{RakSri15} we obtain (just as in the proof of Theorem 18 in \cite{RakSri17}) that for any class $\G \subseteq [0,1]^{\Z}$,
\begin{align*}
 P\left(\sup_{g \in \G} \sum_{t=1}^T \left\{  (g(Z_t) - \mathbb{E}_{t-1}[g(Z_t)]) - \frac{\epsilon}{2} \mathbb{E}_{Z'_t}(g(Z_t) - g(Z'_t))^2  \right\} > \mathcal{R}_T(\G,\epsilon)+ \theta\right) \le 3 \exp(- \epsilon \theta)
\end{align*}
Since $ \mathbb{E}_{Z'_t}(g(Z_t) - g(Z'_t))^2 \le  \mathbb{E}_{Z'_t}[g(Z'_t)^2] + g(Z_t)^2 \le \left( \mathbb{E}_{Z'_t}[g(Z'_t)] + g(Z_t)\right) = \left( \mathbb{E}_{t-1}[g(Z_t)] + g(Z_t)\right)$: 
\begin{align*}
 P\left(\sup_{g \in \G} \sum_{t=1}^T \left\{  (g(Z_t) - \mathbb{E}_{t-1}[g(Z_t)]) - \epsilon\left(\mathbb{E}_{t-1}[g(Z_t)] + g(Z_t) \right)  \right\} > \mathcal{R}_T(\G,\epsilon)+ \theta\right) \le 3 \exp(- \epsilon \theta)
\end{align*}
Hence we conclude the tail bound:
\begin{align*}
 P\left(\sup_{f \in \F} \sum_{t=1}^T \left\{   \tilde{\ell}^t_{f(x_t)} - \mathbb{E}_{t-1}[\tilde{\ell}^t_{f(x_t)}] - \epsilon  \left( (\mathbb{E}_{t-1}[\tilde{\ell}^t_{f(x_t)}]) +  \tilde{\ell}^t_{f(x_t)}\right)  \right\} > \frac{1}{\gamma} \mathcal{R}_T(\G,\epsilon)+ \theta\right) \le 3 \exp(- \epsilon \gamma \theta)
\end{align*}
Now further noting that for our estimate, $ \mathbb{E}_{t-1}[\tilde{\ell}^t_{f(x_t)}] \le \ell^t_{f(x_t)}$, we conclude that,
\begin{align*}
 P\left(\sup_{f \in \F} \sum_{t=1}^T \left\{   (\tilde{\ell}^t_{f(x_t)} - \ell^t_{f(x_t)}) - \epsilon  \left( \ell^t_{f(x_t)} +  \tilde{\ell}^t_{f(x_t)}\right)  \right\} > \frac{1}{\gamma} \mathcal{R}_T(\G,\epsilon)+ \theta\right) \le 3 \exp(- \epsilon \gamma \theta)
\end{align*}
Noting that, $\frac{1}{\gamma}  \mathcal{R}_T(\G,\epsilon) = \mathcal{R}_T( \frac{1}{\gamma} \G,\epsilon \gamma) = \mathcal{R}_T( \F,\epsilon \gamma)$, and setting probability of failure to $\delta$ and rewriting the above statement we obtain that: for any $\delta >0$, w.p. at least $1 - \delta$,
$$
\sup_{f \in \F} \sum_{t=1}^T \left\{   (1 -\eps)\tilde{\ell}^t_{f(x_t)} - (1 + \eps) \ell^t_{f(x_t)})  \right\} >  \mathcal{R}_T(\F,\epsilon \gamma)+ \frac{ \log(3/\delta)}{\gamma \epsilon} 
$$
This concludes the proof.
\end{proof}
\end{document}